\documentclass{article}

\usepackage{amsmath,amsfonts,bm}

\def\eqref#1{equation~\ref{#1}}
\def\1{\bm{1}}

\def\vb{{\bm{b}}}

\def\vx{{\bm{x}}}
\def\vy{{\bm{y}}}
\def\vz{{\bm{z}}}

\def\mB{{\bm{B}}}

\def\mI{{\bm{I}}}

\def\mW{{\bm{W}}}

\def\mY{{\bm{Y}}}
\def\mZ{{\bm{Z}}}

\DeclareMathAlphabet{\mathsfit}{\encodingdefault}{\sfdefault}{m}{sl}
\SetMathAlphabet{\mathsfit}{bold}{\encodingdefault}{\sfdefault}{bx}{n}

\def\gA{{\mathcal{A}}}

\def\gD{{\mathcal{D}}}

\def\gH{{\mathcal{H}}}

\def\gL{{\mathcal{L}}}

\def\gN{{\mathcal{N}}}
\def\gO{{\mathcal{O}}}
\def\gP{{\mathcal{P}}}
\def\gQ{{\mathcal{Q}}}

\def\gS{{\mathcal{S}}}

\def\gU{{\mathcal{U}}}

\newcommand{\R}{\mathbb{R}}

\DeclareMathOperator*{\argmin}{arg\,min}

\usepackage{microtype}
\usepackage{graphicx}
\usepackage{booktabs} % for professional tables

\usepackage{hyperref}
\usepackage{url}
\usepackage{amsfonts}       % blackboard math symbols
\usepackage{nicefrac}       % compact symbols for 1/2, etc.
\usepackage{xcolor}         % colors
\usepackage{caption}
\usepackage{subcaption}
\usepackage{multirow}
\usepackage{rotating}
\usepackage{float}
\usepackage{algorithmic}
\usepackage{bigstrut}
\usepackage{xparse}

\usepackage{listings}
\usepackage[accepted]{icml2023}

\usepackage{amsmath}
\usepackage{amssymb}
\usepackage{mathtools}
\usepackage{amsthm}

\usepackage{thmtools}
\usepackage{thm-restate}

\usepackage[capitalize,noabbrev]{cleveref}

\theoremstyle{plain}
\newtheorem{theorem}{Theorem}[section]

\theoremstyle{definition}

\theoremstyle{remark}

\usepackage[textsize=tiny]{todonotes}

\def\hrz{t:t+H}
\def\lb{t-L:t}
\def\vtau{\bm{\tau}}
\def\Ttrain{T'}

\newcommand{\splitatcommas}[1]{%
  \begingroup
  \begingroup\lccode`~=`, \lowercase{\endgroup
    \edef~{\mathchar\the\mathcode`, \penalty0 \noexpand\hspace{0pt plus 1em}}%
  }\mathcode`,="8000 #1%
  \endgroup
}

\ExplSyntaxOn
\NewDocumentEnvironment{places}{mm}
 {% #1 is the desired width, #2 is the number of photos per line
  \setlength{\tabcolsep}{0pt} % no space between rows
  \dim_set:Nn \l_places_width_dim
   {
    (#1-\ht\strutbox-\dp\strutbox-2pt)/(#2)
   }
  \begin{tabular}{r @{\hspace{2pt}} *{#2}{c}}
 }
 {
  \end{tabular}
 }

\NewDocumentCommand{\place}{mm}
 {% #1 is the name of the place, #2 is the comma separated list of images
  \seq_set_from_clist:Nn \l_places_images_in_seq { #2 }
  \seq_set_map:NNn \l_places_images_out_seq \l_places_images_in_seq { \places_set_image:n {##1} }
  \seq_put_left:Nn \l_places_images_out_seq
   {
    \begin{tabular}{c}\rotatebox[origin=c]{90}{\strut#1}\end{tabular}
   }
  \seq_use:Nn \l_places_images_out_seq { & } \\ \addlinespace
 }

\dim_new:N \l_places_width_dim
\seq_new:N \l_places_images_in_seq
\seq_new:N \l_places_images_out_seq

\cs_new_protected:Nn \places_set_image:n
 {
  \makebox[\l_places_width_dim]
   {
    \begin{tabular}{c}
    \includegraphics[
      width=\l_places_width_dim,
      height=\l_places_width_dim,
      keepaspectratio,
    ]{#1}
    \end{tabular}
   }
 }

\ExplSyntaxOff

\def\shortname{DeepTime}
\def\longname{Learning Deep Time-index Models for Time Series Forecasting}

\icmltitlerunning{{\longname}}

\begin{document}

\twocolumn[
\icmltitle{{\longname}}

% It is OKAY to include author information, even for blind
% submissions: the style file will automatically remove it for you
% unless you've provided the [accepted] option to the icml2023
% package.

% List of affiliations: The first argument should be a (short)
% identifier you will use later to specify author affiliations
% Academic affiliations should list Department, University, City, Region, Country
% Industry affiliations should list Company, City, Region, Country

% You can specify symbols, otherwise they are numbered in order.
% Ideally, you should not use this facility. Affiliations will be numbered
% in order of appearance and this is the preferred way.
\icmlsetsymbol{equal}{*}

\begin{icmlauthorlist}
\icmlauthor{Gerald Woo}{SFRA,SMU}
\icmlauthor{Chenghao Liu}{SFRA}
\icmlauthor{Doyen Sahoo}{SFRA}
\icmlauthor{Akshat Kumar}{SMU}
\icmlauthor{Steven Hoi}{SFRA}
\end{icmlauthorlist}

\icmlaffiliation{SFRA}{Salesforce Research Asia}
\icmlaffiliation{SMU}{School of Computing and Information Systems, Singapore Management University}

\icmlcorrespondingauthor{Gerald Woo}{gwoo@salesforce.com}
\icmlcorrespondingauthor{Chenghao Liu}{chenghao.liu@salesforce.com}

% You may provide any keywords that you
% find helpful for describing your paper; these are used to populate
% the "keywords" metadata in the PDF but will not be shown in the document
\icmlkeywords{Machine Learning, ICML, Time Series, Forecasting, Deep Learning, Meta-optimization, Time-index Model, Time-index, Game-changing}

\vskip 0.3in
]

% this must go after the closing bracket ] following \twocolumn[ ...

% This command actually creates the footnote in the first column
% listing the affiliations and the copyright notice.
% The command takes one argument, which is text to display at the start of the footnote.
% The \icmlEqualContribution command is standard text for equal contribution.
% Remove it (just {}) if you do not need this facility.

\printAffiliationsAndNotice{}  % leave blank if no need to mention equal contribution
% \printAffiliationsAndNotice{\icmlEqualContribution} % otherwise use the standard text.

\begin{abstract}
Deep learning has been actively applied to time series forecasting, leading to a deluge of new methods, belonging to the class of \textit{historical-value} models. Yet, despite the attractive properties of \textit{time-index} models, such as being able to model the continuous nature of underlying time series dynamics, little attention has been given to them. Indeed, while naive deep time-index models are far more expressive than the manually predefined function representations of classical time-index models, they are inadequate for forecasting, being unable to generalize to unseen time steps due to the lack of inductive bias. In this paper, we propose {\shortname}, a meta-optimization framework to learn deep time-index models which overcome these limitations, yielding an efficient and accurate forecasting model. Extensive experiments on real world datasets in the long sequence time-series forecasting setting demonstrate that our approach achieves competitive results with state-of-the-art methods, and is highly efficient. Code is available at \url{https://github.com/salesforce/DeepTime}.
\end{abstract}
\section{Introduction}
\label{sec:intro}
Time series forecasting has important applications across business and scientific domains, such as demand forecasting \citep{carbonneau2008application}, capacity planning and management \citep{kim2003financial}, and anomaly detection \citep{laptev2017time}. There are two typical approaches to time series forecasting -- historical-value, and time-index models. Historical-value models predict future time step(s) as a function of past observations, \(\hat{\vy}_{t+1} = f(\vy_t, \vy_{t-1}, \ldots)\) \cite{benidis2020neural}, while time-index models, a less studied approach, are defined to be models whose predictions are \textit{purely} functions of the corresponding time-index features at future time-step(s), \(\hat{\vy}_{t+1} = f(\vtau_{t+1})\) (see \cref{fig:paradigm} for a visual comparison). 
Historical-value models have been widely used due to their simplicity. However, they can only model temporal relationships at the data sampling frequency. This is an issue since time series observations available to us tend to have a much lower resolution (sampled at a lower frequency) than the underlying dynamics \citep{gong2015discovering, gong2017causal}, leading historical-value models to be prone to capturing spurious correlations. On the other hand, time-index models intrinsically avoid this problem, directly modeling the mapping from time-index features to predictions in continuous space and learning signal representations which change smoothly and correlate with each other in continuous space.\looseness=-1
\begin{figure}[t]
    \centering
    \begin{subfigure}[b]{\columnwidth}
        \centering
        \includegraphics[width=\textwidth]{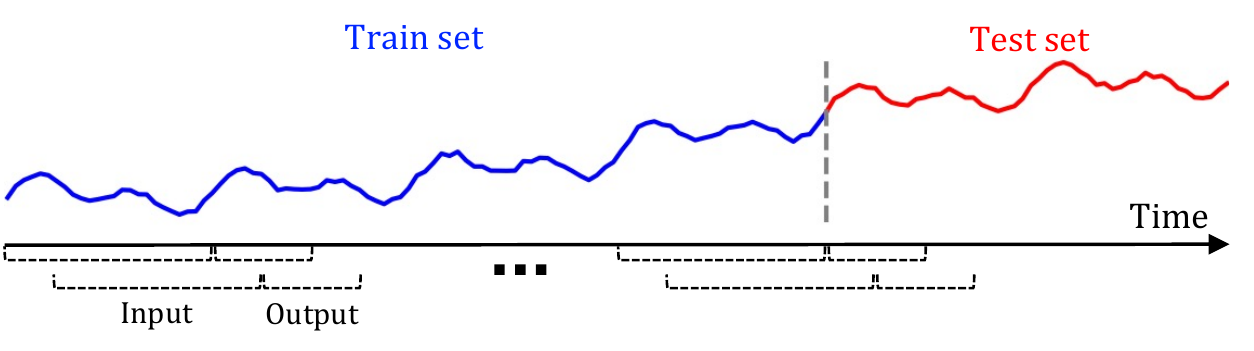}
        \vskip -0.05in
        \caption{Historical-value Models}
    \label{subfig:historical-value-models}
    \end{subfigure}
    \begin{subfigure}[b]{\columnwidth}
        \centering
        \includegraphics[width=\textwidth]{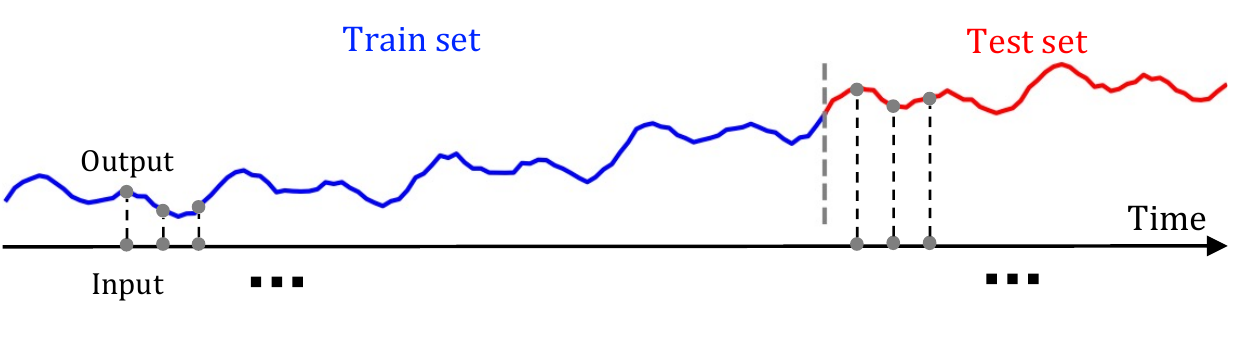}
        \vskip -0.05in
        \caption{Time-index Models}
        \label{subfig:time-index-models}
    \end{subfigure}
    \vskip -0.1in
    \caption{
        Visual comparison between the two paradigms of historical-value, and time-index models.
        After the models are trained on the training set, forecasts are made over the test set, which may be arbitrarily long.
        Historical-value models make predictions conditioning on an input lookback window, whereas time-index models do not make use of new incoming information during the test phase.
    }
    \label{fig:paradigm}
    \vskip -0.15in
\end{figure}

Classical time-index models \cite{taylor2018forecasting, hyndman2018forecasting, ord2017principles} rely on predefined function forms to generate predictions. They optionally follow the structural time-series  formulation \cite{harvey1993structural}, \(\vy_t = g(\vtau_t) + s(\vtau_t) + h(\vtau_t)\), where \(g, s, h\) represent trend, periodic, and holiday components respectively. For example, \(g\) could be predefined as a linear, polynomial, or even a piece-wise linear function. While these functions are simple and easy to learn, they have limited capacity, are unable to fit more complex time series, and such predefined function forms may be too stringent of an assumption which may not hold in practice. While it is possible to perform model selection across various function forms, this requires either strong domain expertise, or computationally heavy cross-validation across a large set of parameters. \looseness=-1

\begin{figure}[t]
    \centering
    \includegraphics[width=\columnwidth]{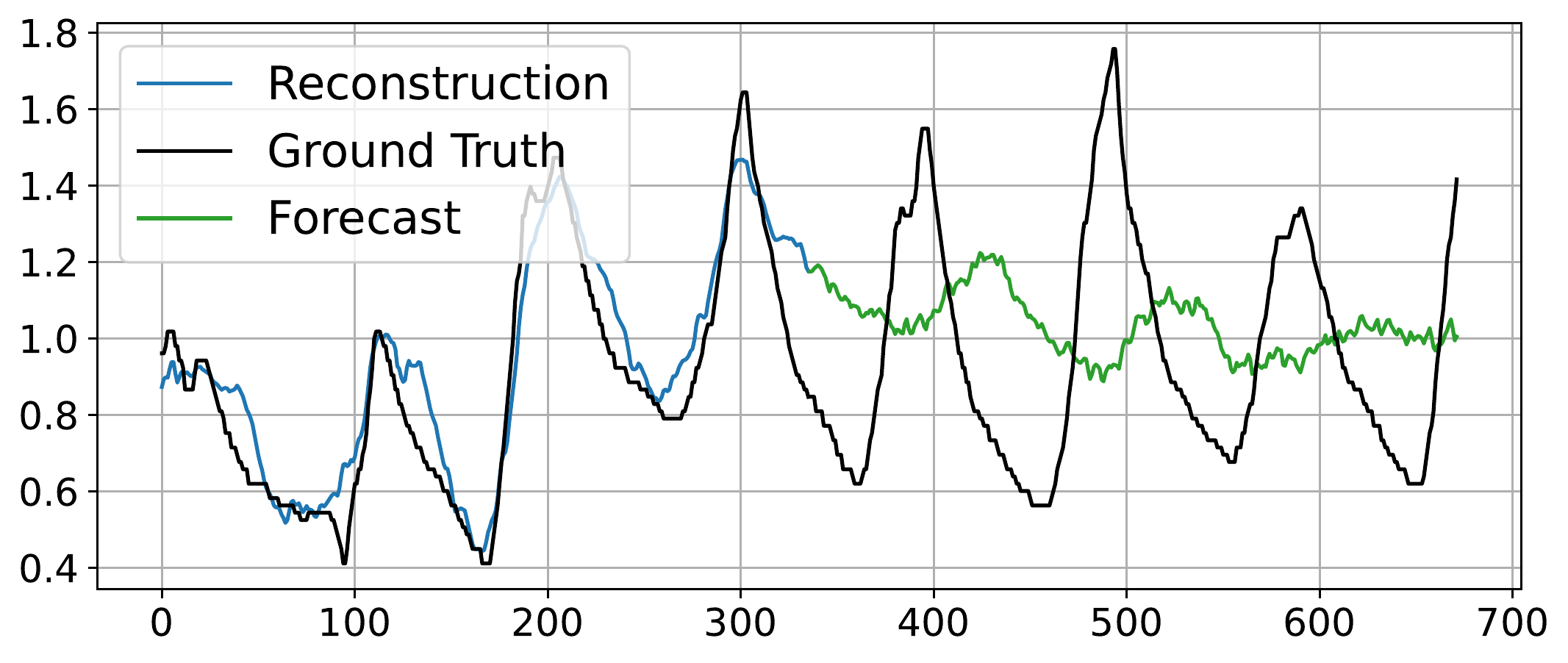}
    \vspace{-0.2in}
    \caption{
    Ground truth time series and predictions from a vanilla deep time-index model. 
    The reconstruction of historical training data as well as forecast over the horizon is visualized.
    Observe that it overfits to historical data and is unable to extrapolate.
    This model corresponds to +Local in \cref{tab:ablation-variants} of our ablations.}
    \label{fig:extrapolate}
    \vspace{-0.2in}
\end{figure}
Deep learning gives a seemingly natural solution to this problem faced by classical time-index models -- parameterize \(f(\vtau_t)\) as a neural network, and learn the function form in a \textit{purely} data-driven manner. While neural networks are extremely expressive with a strong capability to approximate complex functions, we argue that when trained via standard supervised learning, they face the debilitating problem of an inability to extrapolate across the forecast horizon, i.e. unable to generalize to future time step(s). This is visualized in \cref{fig:extrapolate}, where a deep time-index model is able to fit the training data well in the interval \([0, T']\), but it is unable to produce meaningful forecasts for the interval \([T', T]\). This arises due to the lack of extrapolation capability of neural networks when used in time-index models. While classical time-index models have limited expressitivity, the strong inductive biases (e.g. linear trend, periodicity) instructs their predictions over unseen time steps. On the other hand, vanilla deep time-index models do not have such inductive biases, and thus, while they do learn an appropriate function form over the training data, achieving extremely good reconstruction loss, they have little to no generalization capabilities over unseen time steps. \looseness=-1

We now raise the question -- how do we achieve the best of both worlds -- to retain the flexibility and expressiveness of the neural network, as well as to learn inductive bias to guarantee extrapolation capability over unseen time steps from time series data? 
We present a solution, {\shortname}, a meta-optimization framework to learn deep time-index models for time series forecasting. 
Our framework splits the learning process of deep time-index models into two stages, the inner, and outer learning process.
The inner learning process acts as the standard supervised learning process, fitting parameters to recent time steps.
The outer learning process enables the deep time-index model to learn a strong inductive bias for extrapolation from data. 
We summarize our contributions as follows:
\vspace{-0.1in}
\begin{itemize}
    \item We introduce the use of deep time-index models for time series forecasting, proposing a meta-optimization framework to address their shortcomings, making them viable for forecasting.
    \item We introduce a specific function form for deep time-index models by leveraging implicit neural representations \citep{sitzmann2020implicit}, and a novel concatenated Fourier features module to efficiently learn high frequency patterns in time series.
    We also specify an efficient instantiation of the meta-optimization procedure via a closed-form ridge regressor \citep{bertinetto2018metalearning}.
    \item We conduct extensive experiments on the long sequence time series forecasting (LSTF) setting, demonstrating {\shortname} to be extremely competitive with state-of-the-art baselines. At the same time, {\shortname} is highly efficient in terms of runtime and memory.
\end{itemize}
\section{Related Work}
\paragraph{Neural Forecasting} Neural forecasting \cite{benidis2020neural} methods have seen great success in recent times. One related line of research are Transformer-based methods for LSTF \citep{li2019enhancing, zhou2021informer, xu2021autoformer, woo2022etsformer, zhou2022fedformer} which aim to not only achieve high accuracy, but to overcome the vanilla attention's quadratic complexity.
Fully connected methods \citep{oreshkin2020nbeats, challu2022nhits} have also shown success, with \cite{challu2022nhits} introducing hierarchical interpolation and multi-rate data sampling for the LSTF task.
Bi-level optimization in the form of meta-learning and the use of differentiable closed-form solvers has been explored in time series forecasting \citep{grazzi2021meta}, for the purpose of adapting to new time series datasets, where tasks are defined to be the entire time series.

\paragraph{Time-index Models}
Time-index models take as input time-index features such as datetime features to predict the value of the time series at that time step. They have been well explored as a special case of regression analysis \citep{hyndman2018forecasting, ord2017principles}, and many different predictors have been proposed for the classical setting,including linear, polynomial, and piecewise linear trends, and dummy variables indicating holidays. 
Of note, Fourier terms have been used to model periodicity, or seasonal patterns, and is also known as harmonic regression \citep{young1999dynamic}.
Prophet \citep{taylor2018forecasting} is a popular classical approach which uses a structural time series formulation, specialized for business forecasting.
Another classical approach of note are Gaussian Processes \citep{rasmussen2003gaussian, corani2021time} which are non-parametric models, often requiring complex kernel engineering.
\cite{godfrey2017neural} introduced an initial attempt at using time-index based neural networks to fit a time series for forecasting. Yet, their work is more reminiscent of classical methods, manually specifying periodic and non-periodic activation functions, analogous to the representation functions.

\paragraph{Implicit Neural Representations} INRs have recently gained popularity in the area of neural rendering \citep{tewari2021advances}. They parameterize a signal as a continuous function, mapping a coordinate to the value at that coordinate. A key finding was that positional encodings \citep{mildenhall2020nerf, tancik2020fourier} are critical for ReLU MLPs to learn high frequency details, while another line of work introduced periodic activations \citep{sitzmann2020implicit}. 
Meta-learning on via INRs have been explored for various data modalities, typically over images or for neural rendering tasks \citep{sitzmann2020metasdf, tancik2021learned, dupont2021generative}, using both hypernetworks and optimization-based approaches. \cite{yuce2021structured} show that meta-learning on INRs is analogous to dictionary learning.
In time series, \cite{jeong2022time} explored using INRs for anomaly detection, opting to make use of periodic activations and temporal positional encodings. 

\begin{figure*}[ht]
\centering
\begin{subfigure}[t]{.175\textwidth}
    \centering
    \includegraphics[width=\textwidth]{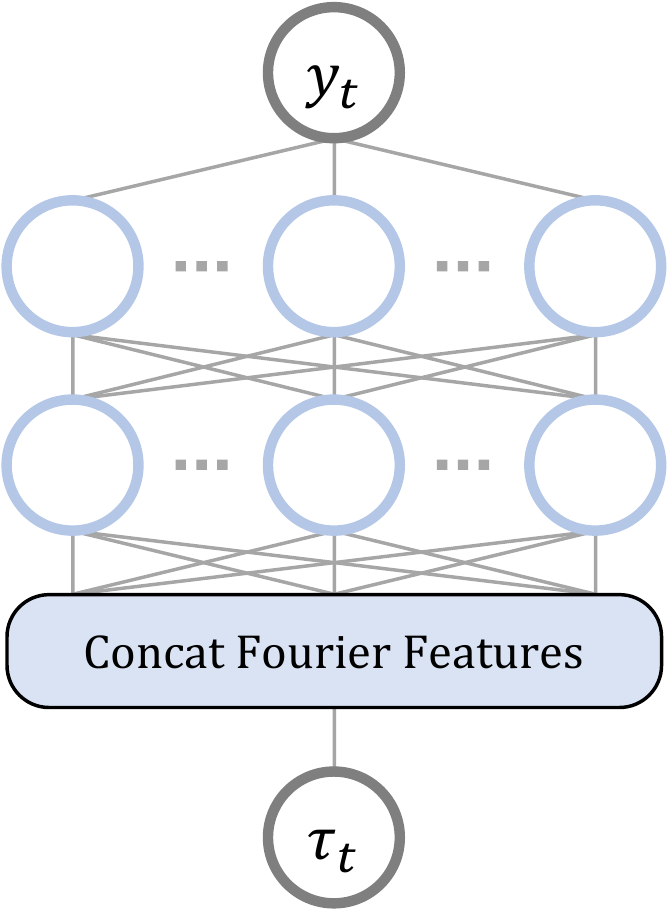}
    \caption{Deep Time-index Model Architecture}
    \label{subfig:model}
\end{subfigure}
\hspace{0.1in}
\begin{subfigure}[t]{.8\textwidth}
    \centering
    \includegraphics[width=\textwidth]{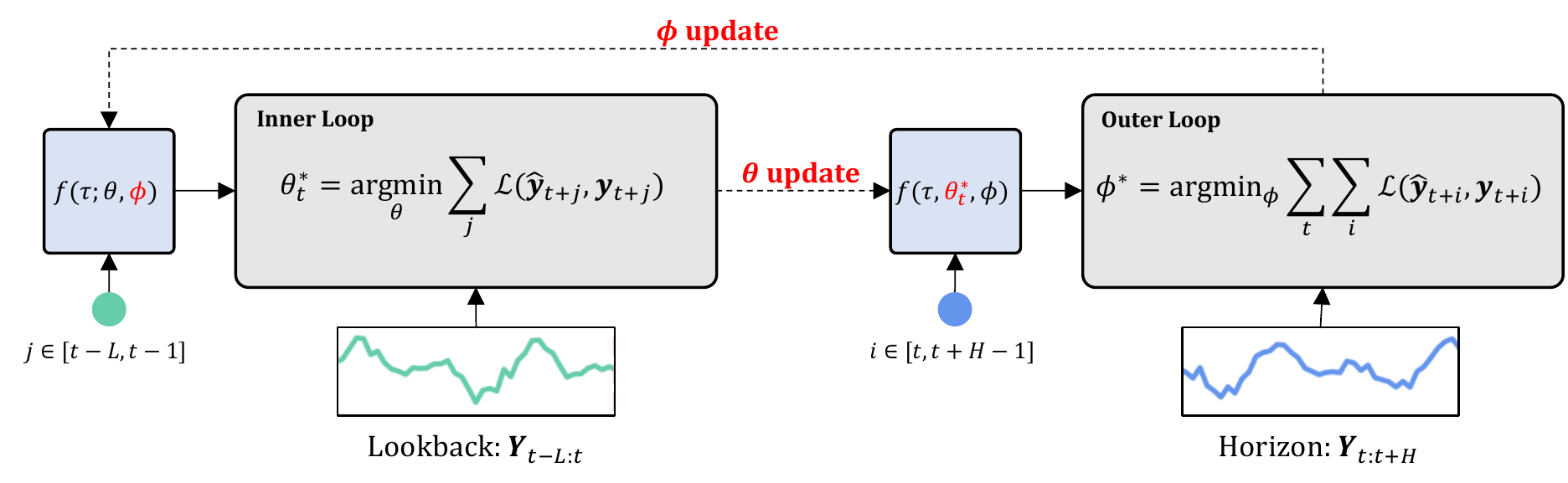}
    \caption{Meta-optimization Framework}
    \label{subfig:meta-optimization-framework}
\end{subfigure}
\vspace{-0.1in}
\caption{
Left: Our proposed deep time-index model has a simple overall architecture, comprising a concatenated Fourier features, and several layers of MLPs.
Right: The meta-optimization framework aims to learn meta parameters of deep time-index models by following a bi-level optimization problem formulation (\cref{eq:outer-loop,eq:inner-loop}). The inner loop is optimized over the lookback window, yielding the optimal base parameters for a locally stationary distribution, \(\theta_t^*\). Each lookback window has it's own optimal base parameters, which are then used to predict the immediate following forecast horizon. This composes the outer loop, which optimizes the meta parameters over the entire time series, yielding an inductive bias which enables extrapolation over forecast horizons, given the optimal base parameters.
}
\vspace{-0.1in}
\label{fig:method}
\end{figure*}
\section{{\shortname}}
In this section, we first formally describe the time series forecasting problem, and how to use time-index models for this problem setting.
Next, we introduce the model architecture specifics for deep time-index models.
Finally, we introduce the generic form of our meta-optimization framework, then specifically, how to use of a differentiable closed-form ridge regression module to perform the meta-optimization efficiently.
Pseudocode of {\shortname} is available in \cref{app:pseudocode}.

\paragraph{Problem Formulation}
In time series forecasting, we consider a time series dataset \((\vy_1, \vy_2, \ldots, \vy_T)\), where \(\vy_t \in \R^m\) is the \(m\)-dimension observation at time \(t\).
Consider a train-test split such that the range \((1, \ldots, \Ttrain)\) is considered to be the training set and the range \((\Ttrain+1, \ldots, T)\) is the test set, where \(T - \Ttrain \geq H\). The goal is to construct point forecasts over a horizon of  length \(H\), over the test set, \(\hat{\mY}_{\hrz} = [\hat{\vy}_t; \ldots; \hat{\vy}_{t+H-1}], \forall t = \Ttrain + 1, \Ttrain + 2, \ldots, \Ttrain - H+1\). 
A time-index model, \(f: \R \to \R^m, f: \vtau_t \mapsto \hat{\vy}_t\), achieves this by minimizing a reconstruction loss \(\gL: \R^m \times \R^m \to \R\) over the training set, where \(\vtau_t\) is a time-index feature for which values are known for all time steps. Then, we can query it over the test set to obtain forecasts, \(\hat{\mY}_{\hrz} = f(\vtau_{\hrz})\).

\subsection{Deep Time-index Model Architecture}
INRs \citep{sitzmann2020implicit} are a class of coordinate-based models, mapping coordinates to values, which have been extensively studied. 
Time-index models are a case of 1d coordinate-based models, thus, we leverage this existing class of models, which are essentially a stack of multi-layered perceptrons as our proposed deep time-index model architecture.
Visualized in \cref{subfig:model}, a \(K\)-layered, ReLU \citep{nair2010rectified} INR is a function \(f_{\theta}: \R^c \to \R^m\) where:\looseness=-1
\begin{align}
    \vz^{(0)} & = \vtau \nonumber \\
    \vz^{(k+1)} & = \max(0, \mW^{(k)} \vz^{(k)} + \vb^{(k)}), \quad k = 0, \ldots, K-1 \nonumber \\
    f_{\theta}(\vtau) & = \mW^{(K)} \vz^{(K)} + \vb^{(K)} \label{eq:inr}
\end{align}
where \(\vtau \in \R^c\) are time-index features. 
MLPs have shown to experience difficulty in learning high frequency functions, this problem known as ``spectrial bias'' \citep{rahaman2019spectral, basri2020frequency}. Coordinate-based methods suffer from this issue in particular when trying to represent high frequency content present in the signal.
\citet{tancik2020fourier} introduced a random Fourier features layer which allows INRs to fit to high frequency functions, by modifying \(\vz^{(0)} = \gamma(\vtau) = [\sin(2 \pi \mB \vtau), \cos(2 \pi \mB \vtau)]^T\), where each entry in \(\mB \in \R^{d/2 \times c}\) is sampled from \(\gN(0, \sigma^2)\) with \(d\) is the hidden dimension size of the INR and \(\sigma^2\) is the scale hyperparameter. \([\cdot, \cdot]\) is a row-wise stacking operation.
\vspace{-0.1in}

\paragraph{Concatenated Fourier Features}
While the random Fourier features layer endows INRs with the ability to learn high frequency patterns, one major drawback is the need to perform a hyperparameter sweep for each task and dataset to avoid over or underfitting. We overcome this limitation with a simple scheme of concatenating multiple Fourier basis functions with diverse scale parameters, i.e. \(\splitatcommas{\gamma(\vtau) = [\sin(2 \pi \mB_1 \vtau), \cos(2 \pi \mB_1 \vtau), \ldots, \sin(2 \pi \mB_S \vtau), \cos(2 \pi \mB_S \vtau)]^T}\), where elements in \(\mB_s \in \R^{d/2 \times c}\) are sampled from \(\gN(0, \sigma_s^2)\), and \(\mW^{(0)} \in \R^{d \times Sd}\). We perform an analysis in \cref{subsec:ablation} and show that the performance of our proposed concatenated Fourier features (CFF) does not significantly deviate from the setting with the optimal scale parameter obtained from a hyperparameter sweep.

\subsection{Meta-optimization Framework}
\label{subsec:meta-optimization}
Explained in \cref{sec:intro}, vanilla deep time-index models are unable to perform forecasting due to their failure in extrapolating beyond observed time-indices. 
Formally, the original hypothesis class of time-index model is denoted \(\gH_{\mathrm{INR}} = \left \{ f(\tau; \theta) \mid \theta \in \Theta \right \}\), where \(\Theta\) is the parameter space. The original hypothesis class is too expressive, providing no guarantees that training on the lookback window leads to good extrapolation. 
To solve this, our meta-optimization framework learns an inductive bias for the deep time-index model from data. 
Firstly, rather than using the entire training set in a naive supervised learning setting, whereby older training points provide no additional benefit in learning a time-index model for extrapolation, we split the time series into lookback window and forecast horizon pairs. 
Let the \(L\) time steps preceding the forecast horizon at time step \(t\), be the lookback window, \(\mY_{\lb} = [\vy_{t-L}; \ldots; \vy_{t-1}]^T \in \R^{L \times m}\). 
Next, consider the case where we split the model parameters into two, possibly overlapping subsets, \(\phi \in \Phi\) and \(\theta \in \Theta\), known as the meta and base parameters, respectively, where \(\Phi\) is the parameter space of the meta parameters.
The meta parameters are responsible for learning the inductive bias from multiple lookback window and forecast horizon pairs from the training data, while the base parameters aim to learn and adapt quickly to the lookback window at test time.
Thus, we aim to encode an inductive bias in \(\phi\), \textbf{learned} to enable extrapolation across the forecast horizon when the base parameters adapt corresponding lookback window, resulting in \(\gH_{\mathrm{{Meta}}} = \left \{ f(\tau; \theta, \phi^*) \mid \theta \in \Theta \right \}\).
This is achieved by formulating the bi-level optimization problem:
\looseness=-2
\begin{align}
    \phi^* & = \argmin_{\phi} \sum_{t=L+1}^{T-H+1} \; \sum_{i=0}^{H-1} \gL(f(\vtau_{t+i}; \theta^*_t, \phi), \vy_{t+i}) \label{eq:outer-loop} \\
    s.t. \quad \theta^*_t & = \argmin_{\theta} \sum_{j=-L}^{-1} \gL(f(\vtau_{t+j}; \theta, \phi), \vy_{t+j}) \label{eq:inner-loop}
\end{align}
Illustrated in \cref{subfig:meta-optimization-framework}, \cref{eq:outer-loop} represents the outer loop, and \cref{eq:inner-loop}, the inner loop. The first summation in the outer loop over index \(t\) represents iterating over the time steps of the dataset, and the second summation over index \(i\) represents each time step of the forecast horizon. The summation in the inner loop over index \(j\) represents each time step of the lookback window.

\paragraph{Fast and Efficient Meta-optimization}
Following the above generic formulation of the meta-optimization framework to learn deep time-index models, we now describe a specific instantiation of the framework which enables both training and forecasting at test time to be fast and efficient. 
Similar bi-level optimization problems have been explored \cite{ravi2016optimization, finn2017model} and one naive approach is to directly backpropagate through inner gradient steps. However, such methods are highly inefficient, have many additional hyperparameters, and are instable during training \cite{antoniou2018how}. 
Instead, to achieve speed and efficiency, we specify that the base parameters consist of only the last layer of the INR, while the rest of the INR are meta parameters. Thus, the inner loop optimization only applies to this last layer. This transforms the inner loop optimization problem into a simple ridge regression problem for the case of mean squared error loss, having a simple analytic solution to replace the otherwise complicated non-linear optimization problem \cite{bertinetto2018metalearning}.

Formally, for a \(K\)-layered model, \(\phi = \{\mW^{(0)}, \vb^{(0)}, \ldots, \mW^{(K-1)}, \vb^{(K-1)}, \lambda\}\) are the meta parameters and \(\theta = \{\mW^{(K)}\}\) are the base parameters, following notation from \cref{eq:inr}. 
Then let \(g_{\phi}: \R \to \R^d\) be the meta learner where \(g_{\phi}(\vtau) = \vz^{(K)}\). 
For a lookback-horizon pair, \((\mY_{\lb}, \mY_{\hrz})\), the features of the lookback window obtained from the meta learner is denoted \(\mZ_{\lb} = [g_{\phi}(\vtau_{t-L}); \ldots; g_{\phi}(\vtau_{t-1})]^T \in \R^{L \times d}\), where \([\cdot ; \cdot]\) is a column-wise concatenation operation. The inner loop thus solves the optimization problem:
\begin{align}
    \mW^{(K)*}_t & = \argmin_{\mW} ||\mZ_{\lb}\mW - \mY_{\lb}||^2 + \lambda ||\mW||^2 \nonumber \\
    & = (\mZ_{\lb}^T\mZ_{\lb} + \lambda \mI)^{-1} \mZ_{\lb}^T \mY_{\lb} \label{eq:diff-closed-form}
\end{align}
Now, let \(\mZ_{\hrz} = [g_{\phi}(\vtau_{t}); \ldots; g_{\phi}(\vtau_{t+H-1})]^T \in \R^{H \times d}\) be the features of the forecast horizon. Then, our predictions are \(\hat{\mY}_{\hrz} = \mZ_{\hrz}\mW_t^{(K)*}\). This closed-form solution is differentiable, which enables gradient updates on the parameters of the meta learner, \(\phi\).
A bias term can be included for the closed-form ridge regressor by appending a scalar 1 to the feature vector \(g_{\phi}(\vtau)\).
The end result of training {\shortname} on a dataset is the restricted hypothesis class \(\gH_\mathrm{{\shortname}} = \left \{ g_{\phi^*}(\vtau)^T \mW^{(K)} \mid \mW^{(K)} \in \R^{d \times m} \right \} \).
Finally, we propose to use relative time-index features, \(\vtau_{t+i} = \frac{i+L}{L+H-1}\) for \(i = -L, -L + 1, \ldots, H-1\), i.e. a \([0,1]\)-normalized time-index.

\section{Theoretical Analysis}
In the following, we derive a generalization bound of {\shortname} under the PAC-Bayes framework \cite{amit2018meta}.
Give a long time series, suppose we can split it into \(n\) instances (each has a length \(L\) lookback window and length \(H\) forecast horizon) for training. Then the \(k\)-th instance is denoted \(\gS_k = \{z_{k-L}, \ldots, z_k, \ldots, z_{k+H-1}\}\), where \(z_t = (\vtau_t, \vy_t)\). The PAC-Bayes framework for our proposed meta-optimization framework considers a Bayesian setting of {\shortname}, having prior and posterior distributions of the meta and base parameters. The generalization bound is presented below, with the detailed proof in \cref{app:pac-bayes}.

\begin{theorem}{(Generalization Bound)}
\label{thm:bound}
Let \(\mathcal{Q}, Q\) be  arbitrary distribution of \(\phi, \theta\), which are defined in \cref{eq:outer-loop} and \cref{eq:inner-loop}, and \(\mathcal{P}, P\) be the prior distribution of \(\phi, \theta\). Then for any $c_1,c_2>0$ and any $\delta\in(0,1]$, with probability at least $1-\delta$, the following inequality holds uniformly for all hyper-posterior distributions $\gQ$,
\begin{align}
er(\mathcal{Q}) \leq \frac{c_1c_2}{(1-e^{-c_1})(1-e^{-c_2})}\cdot \frac{1}{n}\sum_{k=1}^{n}\hat{er}(\mathcal{Q}, \gS_k) \nonumber\\
+ \frac{c_1}{1-e^{-c_1}} \cdot  \frac{\mathrm{KL}(\gQ||\gP)+\log \frac{2}{\delta}}{nc_1} \nonumber\\
+ \frac{c_1c_2}{(1-e^{-c_2})(1-e^{-c_1})} \cdot \frac{\mathrm{KL}(Q||P)+\log \frac{2n}{\delta}}{(H+L)c_2}
\end{align}
where \(er(\gQ)\) and \(\hat{er}(\gQ, \gS_k)\) are the generalization error and training error of {\shortname}, respectively.
\end{theorem}

\cref{thm:bound} states that the expected generalization error of DeepTime is bounded by the empirical error plus two complexity terms.
The first term represents the complexity between the meta distributions, \(\gQ, \gP\), as well as between instances, converging to zero if we observe an infinitely long time-series ($n \rightarrow \infty$).
The second term represents the complexity between the base distributions, \(Q, P\), and of each instance, or equivalently, the lookback window and forecast horizon. This term converges to zero when there are a sufficient number of time steps in each instance ($H + L \rightarrow \infty$).\looseness=-2
\newpage
\section{Experiments}
\label{sec:experiments}
We evaluate {\shortname} on both synthetic, and real-world datasets. We ask the following questions: 
(i) Is {\shortname}, trained on a family of functions following the same parametric form, able to perform extrapolation on unseen functions?
(ii) How does {\shortname} compare to other forecasting models on real-world data?
(iii) What are the key contributing factors to the good performance of {\shortname}? 

\begin{figure*}[t]
\centering
\begin{subfigure}[b]{.32\textwidth}
    \centering
    \includegraphics[width=\textwidth]{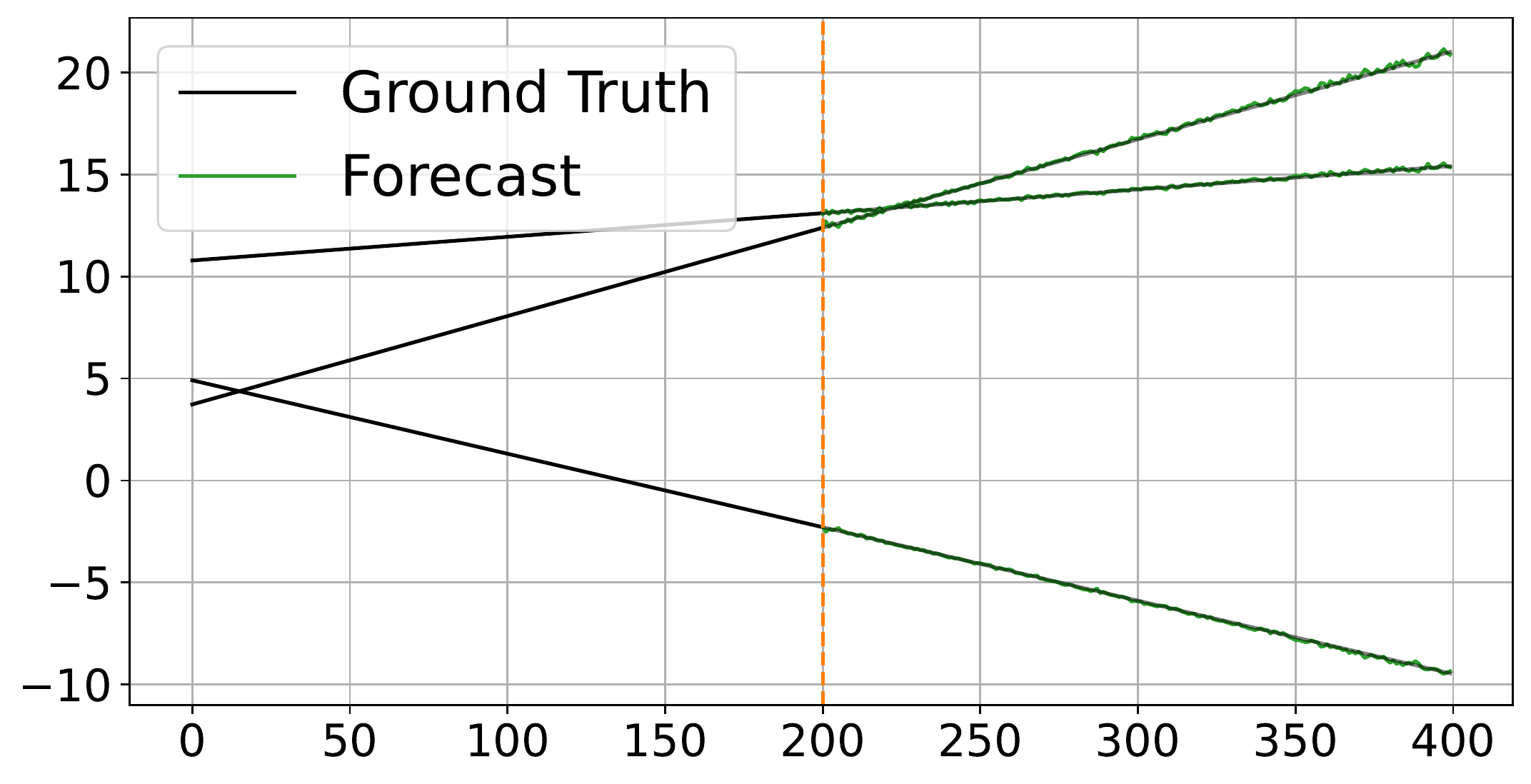}
    \caption{Linear}
    \label{subfig:synthetic-linear}
\end{subfigure}
\begin{subfigure}[b]{.32\textwidth}
    \centering
    \includegraphics[width=\textwidth]{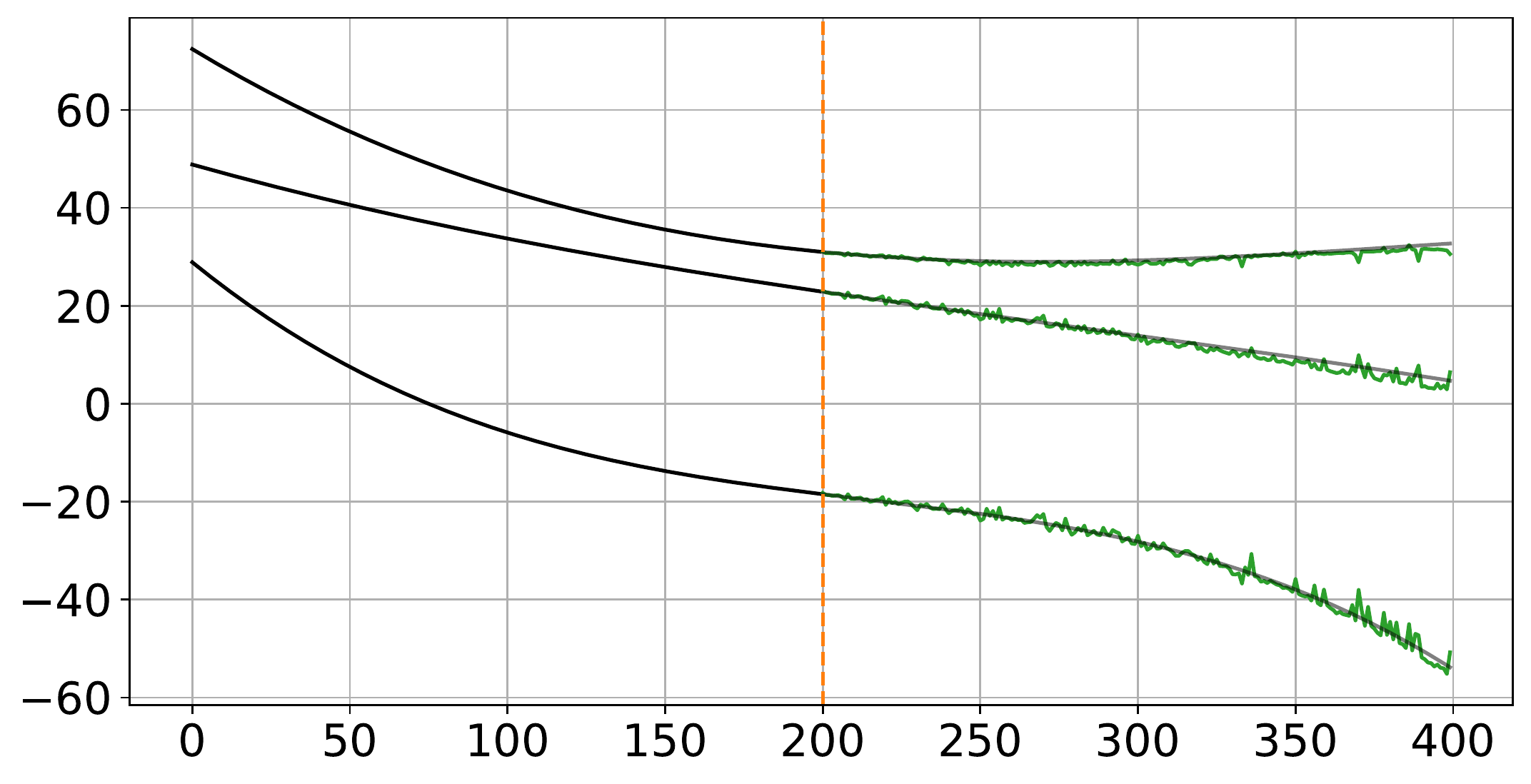}
    \caption{Cubic}
    \label{subfig:synthetic-cubic}
\end{subfigure}
\begin{subfigure}[b]{.32\textwidth}
    \centering
    \includegraphics[width=\textwidth]{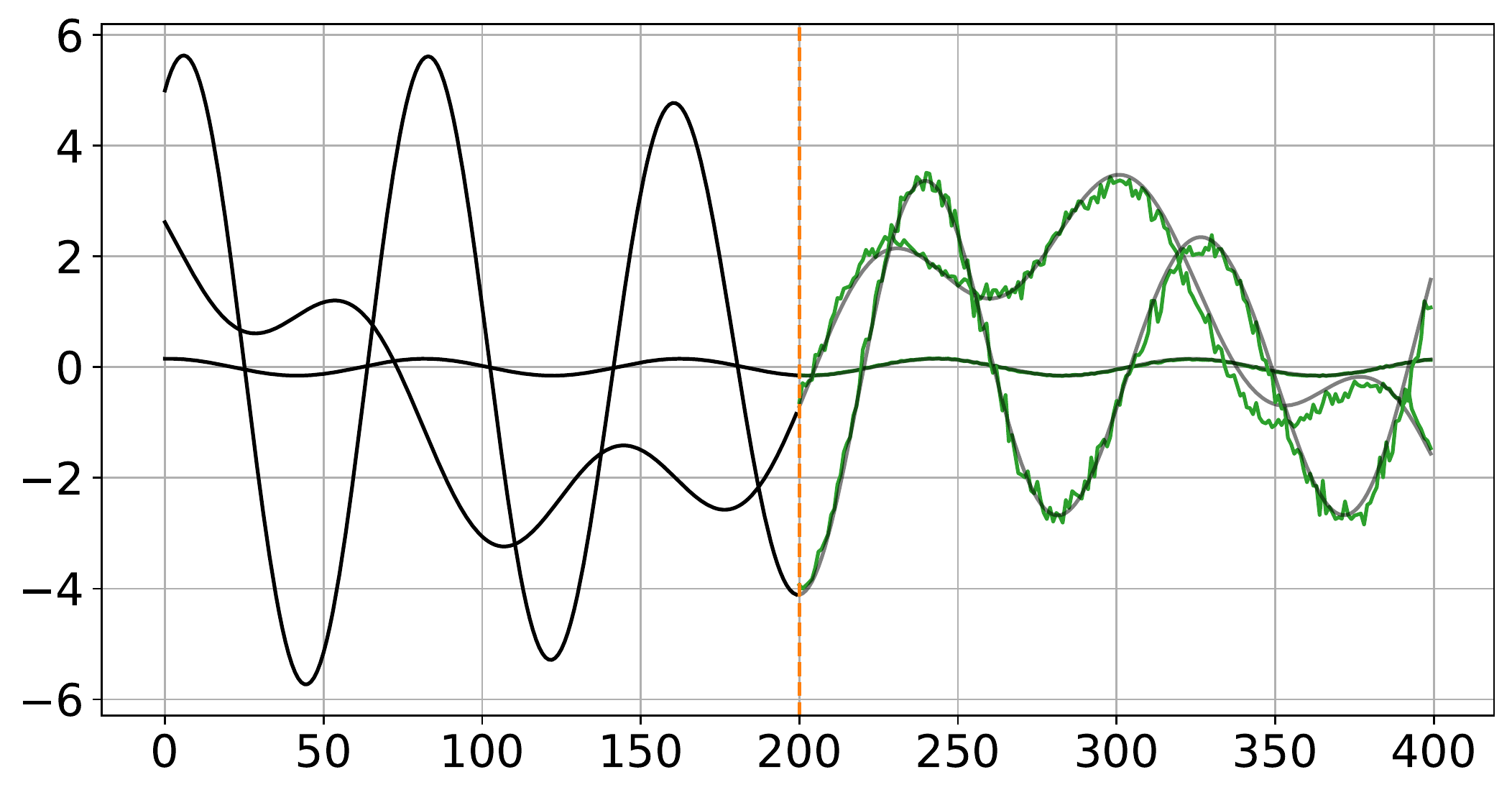}
    \caption{Sum of Sinusoids}
    \label{subfig:synthetic-sine}
\end{subfigure}
\caption{
Predictions of {\shortname} on three \textbf{unseen} functions for each function class. The \textcolor{orange}{orange line} represents the split between lookback window and forecast horizon.
}
\label{fig:synthetic}
\end{figure*}
\begin{table*}[t]
  \centering
  \caption{Multivariate forecasting benchmark on long sequence time-series forecasting. Best results are highlighted in \textbf{bold}, and second best results are \underline{underlined}.}
  \vskip -0.1in
  \resizebox{\textwidth}{!}{
% Table generated by Excel2LaTeX from sheet 'Main - Multivar'
\begin{tabular}{c|c|cccccccccccccccccc}
\multicolumn{1}{r}{} & \multicolumn{1}{r}{} &       &       &       &       &       &       &       &       &       &       &       &       &       &       &       &       &       &  \\
\midrule
\multicolumn{2}{c}{Methods} & \multicolumn{2}{c}{{\shortname}} & \multicolumn{2}{c}{NS Transformer} & \multicolumn{2}{c}{N-HiTS} & \multicolumn{2}{c}{ETSformer} & \multicolumn{2}{c}{FEDformer} & \multicolumn{2}{c}{Autoformer} & \multicolumn{2}{c}{Informer} & \multicolumn{2}{c}{LogTrans} & \multicolumn{2}{c}{GP} \\
\midrule
\multicolumn{2}{c}{Metrics} & MSE   & MAE   & MSE   & MAE   & MSE   & MAE   & MSE   & MAE   & MSE   & MAE   & MSE   & MAE   & MSE   & MAE   & MSE   & MAE   & MSE   & MAE \\
\midrule
\multirow{4}[2]{*}{\begin{sideways}ETTm2\end{sideways}} & 96    & \textbf{0.166} & \underline{0.257} & 0.192 & 0.274 & \underline{0.176} & \textbf{0.255} & 0.189 & 0.280 & 0.203 & 0.287 & 0.255 & 0.339 & 0.365 & 0.453 & 0.768 & 0.642 & 0.442 & 0.422 \\
      & 192   & \textbf{0.225} & \textbf{0.302} & 0.280 & 0.339 & \underline{0.245} & \underline{0.305} & 0.253 & 0.319 & 0.269 & 0.328 & 0.281 & 0.340 & 0.533 & 0.563 & 0.989 & 0.757 & 0.605 & 0.505 \\
      & 336   & \textbf{0.277} & \textbf{0.336} & 0.334 & 0.361 & \underline{0.295} & \underline{0.346} & 0.314 & 0.357 & 0.325 & 0.366 & 0.339 & 0.372 & 1.363 & 0.887 & 1.334 & 0.872 & 0.731 & 0.569 \\
      & 720   & \textbf{0.383} & \textbf{0.409} & 0.417 & 0.413 & \underline{0.401} & 0.426 & 0.414 & \underline{0.413} & 0.421 & 0.415 & 0.422 & 0.419 & 3.379 & 1.388 & 3.048 & 1.328 & 0.959 & 0.669 \\
\midrule
\multirow{4}[2]{*}{\begin{sideways}ECL\end{sideways}} & 96    & \textbf{0.137} & \textbf{0.238} & 0.169 & 0.273 & \underline{0.147} & \underline{0.249} & 0.187 & 0.304 & 0.183 & 0.297 & 0.201 & 0.317 & 0.274 & 0.368 & 0.258 & 0.357 & 0.503 & 0.538 \\
      & 192   & \textbf{0.152} & \textbf{0.252} & 0.182 & 0.286 & \underline{0.167} & \underline{0.269} & 0.199 & 0.315 & 0.195 & 0.308 & 0.222 & 0.334 & 0.296 & 0.386 & 0.266 & 0.368 & 0.505 & 0.543 \\
      & 336   & \textbf{0.166} & \textbf{0.268} & 0.200 & 0.304 & \underline{0.186} & \underline{0.290} & 0.212 & 0.329 & 0.212 & 0.313 & 0.231 & 0.338 & 0.300 & 0.394 & 0.280 & 0.380 & 0.612 & 0.614 \\
      & 720   & \textbf{0.201} & \textbf{0.302} & \underline{0.222} & \underline{0.321} & 0.243 & 0.340 & 0.233 & 0.345 & 0.231 & 0.343 & 0.254 & 0.361 & 0.373 & 0.439 & 0.283 & 0.376 & 0.652 & 0.635 \\
\midrule
\multirow{4}[2]{*}{\begin{sideways}Exchange\end{sideways}} & 96    & \textbf{0.081} & \underline{0.205} & 0.111 & 0.237 & 0.092 & 0.211 & \underline{0.085} & \textbf{0.204} & 0.139 & 0.276 & 0.197 & 0.323 & 0.847 & 0.752 & 0.968 & 0.812 & 0.136 & 0.267 \\
      & 192   & \textbf{0.151} & \textbf{0.284} & 0.219 & 0.335 & 0.208 & 0.322 & \underline{0.182} & \underline{0.303} & 0.256 & 0.369 & 0.300 & 0.369 & 1.204 & 0.895 & 1.040 & 0.851 & 0.229 & 0.348 \\
      & 336   & \textbf{0.314} & \textbf{0.412} & 0.421 & 0.476 & 0.371 & 0.443 & \underline{0.348} & \underline{0.428} & 0.426 & 0.464 & 0.509 & 0.524 & 1.672 & 1.036 & 1.659 & 1.081 & 0.372 & 0.447 \\
      & 720   & \textbf{0.856} & \textbf{0.663} & 1.092 & 0.769 & \underline{0.888} & \underline{0.723} & 1.025 & 0.774 & 1.090 & 0.800 & 1.447 & 0.941 & 2.478 & 1.310 & 1.941 & 1.127 & 1.135 & 0.810 \\
\midrule
\multirow{4}[2]{*}{\begin{sideways}Traffic\end{sideways}} & 96    & \textbf{0.390} & \textbf{0.275} & 0.612 & 0.338 & \underline{0.402} & \underline{0.282} & 0.607 & 0.392 & 0.562 & 0.349 & 0.613 & 0.388 & 0.719 & 0.391 & 0.684 & 0.384 & 1.112 & 0.665 \\
      & 192   & \textbf{0.402} & \textbf{0.278} & 0.613 & 0.340 & \underline{0.420} & \underline{0.297} & 0.621 & 0.399 & 0.562 & 0.346 & 0.616 & 0.382 & 0.696 & 0.379 & 0.685 & 0.390 & 1.133 & 0.671 \\
      & 336   & \textbf{0.415} & \textbf{0.288} & 0.618 & 0.328 & \underline{0.448} & \underline{0.313} & 0.622 & 0.396 & 0.570 & 0.323 & 0.622 & 0.337 & 0.777 & 0.420 & 0.733 & 0.408 & 1.274 & 0.723 \\
      & 720   & \textbf{0.449} & \textbf{0.307} & 0.653 & 0.355 & \underline{0.539} & \underline{0.353} & 0.632 & 0.396 & 0.596 & 0.368 & 0.660 & 0.408 & 0.864 & 0.472 & 0.717 & 0.396 & 1.280 & 0.719 \\
\midrule
\multirow{4}[2]{*}{\begin{sideways}Weather\end{sideways}} & 96    & \underline{0.166} & \underline{0.221} & 0.173 & 0.223 & \textbf{0.158} & \textbf{0.195} & 0.197 & 0.281 & 0.217 & 0.296 & 0.266 & 0.336 & 0.300 & 0.384 & 0.458 & 0.490 & 0.395 & 0.356 \\
      & 192   & \textbf{0.207} & \underline{0.261} & 0.245 & 0.285 & \underline{0.211} & \textbf{0.247} & 0.237 & 0.312 & 0.276 & 0.336 & 0.307 & 0.367 & 0.598 & 0.544 & 0.658 & 0.589 & 0.450 & 0.398 \\
      & 336   & \textbf{0.251} & \textbf{0.298} & 0.321 & 0.338 & \underline{0.274} & \underline{0.300} & 0.298 & 0.353 & 0.339 & 0.380 & 0.359 & 0.359 & 0.578 & 0.523 & 0.797 & 0.652 & 0.508 & 0.440 \\
      & 720   & \textbf{0.301} & \textbf{0.338} & 0.414 & 0.410 & \underline{0.351} & \underline{0.353} & 0.352 & 0.388 & 0.403 & 0.428 & 0.419 & 0.419 & 1.059 & 0.741 & 0.869 & 0.675 & 0.498 & 0.450 \\
\midrule
\multirow{4}[2]{*}{\begin{sideways}ILI\end{sideways}} & 24    & 2.425 & 1.086 & 2.294 & \underline{0.945} & \textbf{1.862} & \textbf{0.869} & 2.527 & 1.020 & \underline{2.203} & 0.963 & 3.483 & 1.287 & 5.764 & 1.677 & 4.480 & 1.444 & 2.331 & 1.036 \\
      & 36    & 2.231 & 1.008 & \textbf{1.825} & \textbf{0.848} & \underline{2.071} & \underline{0.969} & 2.615 & 1.007 & 2.272 & 0.976 & 3.103 & 1.148 & 4.755 & 1.467 & 4.799 & 1.467 & 2.167 & 1.002 \\
      & 48    & 2.230 & 1.016 & \textbf{2.010} & \textbf{0.900} & 2.346 & 1.042 & 2.359 & \underline{0.972} & \underline{2.209} & 0.981 & 2.669 & 1.085 & 4.763 & 1.469 & 4.800 & 1.468 & 2.961 & 1.180 \\
      & 60    & \textbf{2.143} & \underline{0.985} & \underline{2.178} & \textbf{0.963} & 2.560 & 1.073 & 2.487 & 1.016 & 2.545 & 1.061 & 2.770 & 1.125 & 5.264 & 1.564 & 5.278 & 1.560 & 3.108 & 1.214 \\
\bottomrule
\end{tabular}%
  }
  \label{tab:main-multi}%
\end{table*}%

\subsection{Experiments on Synthetic Data}
We first consider {\shortname}'s ability to extrapolate on the following functions specified by some parametric form:
(i) the family of linear functions, \(y = ax + b\),
(ii) the family of cubic functions, \(y = ax^3 + bx^2 + cx + d\), and
(iii) sums of sinusoids, \(\sum_j A_j \sin(\omega_j x + \varphi_j)\). 
Parameters of the functions (i.e. \(a, b, c, d, A_j, \omega_j, \varphi_j\)) are sampled randomly (further details in \cref{app:synthetic}) to construct distinct tasks. A total of 400 time steps are sampled, with a lookback window length of 200 and forecast horizon of 200.
\cref{fig:synthetic} demonstrates that {\shortname} is able to perform extrapolation on unseen test functions/tasks after being trained via our meta-optimization formulation. It demonstrates an ability to approximate and adapt, based on the lookback window, linear and cubic polynomials, and even sums of sinusoids.
Next, we evaluate {\shortname} on real-world datasets, against state-of-the-art baselines.

\subsection{Experiments on Real-world Data}
Experiments are performed on 6 real-world datasets -- Electricity Transformer Temperature (ETT), Electricity Consuming Load (ECL), Exchange, Traffic, Weather, and Influenza-like Illness (ILI) with full details in \cref{app:datasets}.
We evaluate the performance of our proposed approach using two metrics, the mean squared error (MSE) and mean absolute error (MAE) metrics. The datasets are split into train, validation, and test sets chronologically, following a 70/10/20 split for all datasets except for \textit{ETTm2} which follows a 60/20/20 split, as per convention. The univariate benchmark selects the last index of the multivariate dataset as the target variable, following previous work \citep{xu2021autoformer}. Preprocessing on the data is performed by standardization based on train set statistics. Hyperparameter selection is performed on only one value, the lookback length multiplier, \(L = \mu * H\), which decides the length of the lookback window. We search through the values \(\mu = [1, 3, 5, 7, 9]\), and select the best value based on the validation loss. 
Further implementation details on {\shortname} are reported in \cref{app:implementation}, and detailed hyperparameters are reported in \cref{app:hyperparams}. Reported results for {\shortname} are averaged over three runs, and standard deviation is reported in \cref{app:sd}. 

\paragraph{Results}
We compare {\shortname} to the following baselines for the multivariate setting, N-HiTS \citep{challu2022nhits}, ETSformer \citep{woo2022etsformer}, Fedformer \citep{zhou2022fedformer} (we report the best score for each setting from the two variants they present), Autoformer \citep{xu2021autoformer}, Informer \citep{zhou2021informer}, LogTrans \citep{li2019enhancing}, Non-stationary (NS) Transformer \citep{liu2022non}, and Gaussian Process (GP) \citep{rasmussen2003gaussian}. 
For the univariate setting, we include additional univariate forecasting models, N-BEATS \citep{oreshkin2020nbeats}, DeepAR \citep{salinas2020deepar}, Prophet \citep{taylor2018forecasting}, and ARIMA. 
Baseline results are obtained from the respective papers.
\cref{tab:main-multi} and \cref{tab:main-uni} (in \cref{app:univar-benchmark} for space) summarizes the multivariate and univariate forecasting results respectively. {\shortname} achieves state-of-the-art performance on 20 out of 24 settings in MSE, and 17 out of 24 settings in MAE on the multivariate benchmark, and also achieves competitive results on the univariate benchmark despite its simple architecture compared to the baselines comprising complex fully connected architectures and computationally intensive Transformer architectures. 

\subsection{Empirical Analysis and Ablation Studies}
\label{subsec:ablation}
\begin{figure}[t]
    \centering
    \includegraphics[width=\columnwidth]{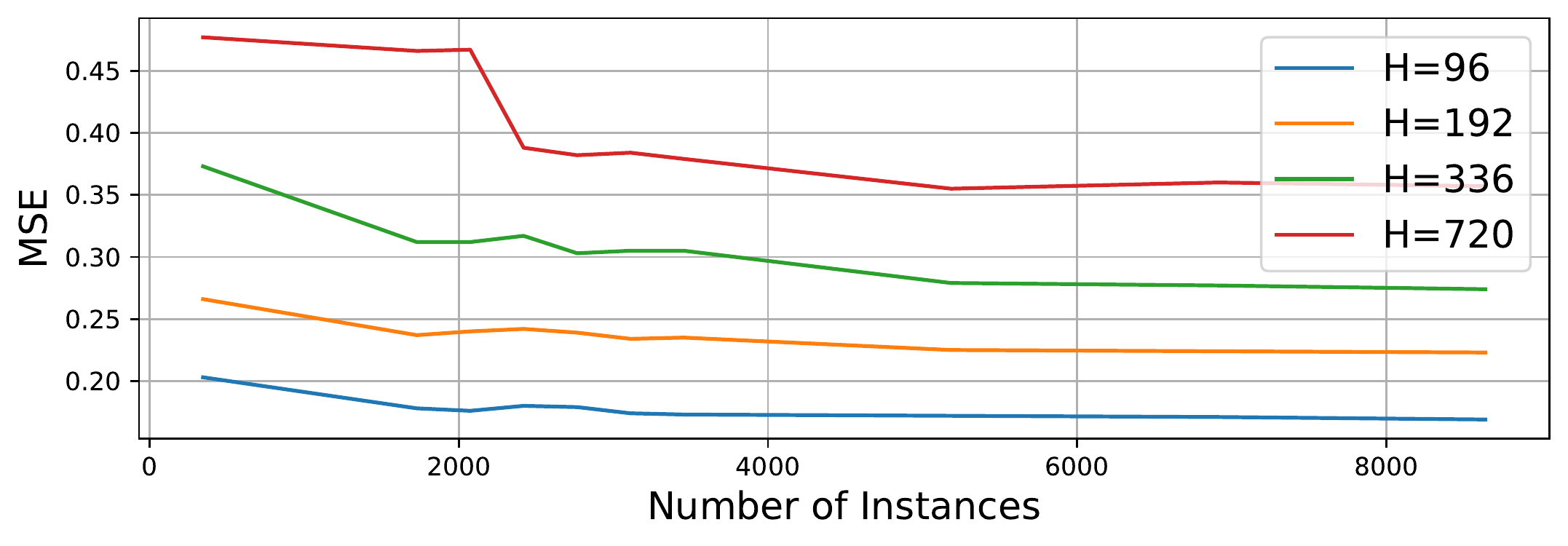}
    \caption{
    Analysis on the number of instances, \(n\). 
    MSE is measured as number of instances increases for multiple horizon lengths.
    Analysis is performed based on the ETTm2 dataset.
    }
    \label{fig:instances}
    \vspace{-0.1in}
\end{figure}
\begin{table}[t]
  \centering
  \caption{
  Analysis on the lookback window length, based on a multiplier on horizon length, \(L = \mu * H\). Results presented on the ETTm2 dataset. Best results are highlighted in \textbf{bold}.}
  \resizebox{\columnwidth}{!}{
    \begin{tabular}{c|cccccccc}
    \toprule
    \multicolumn{1}{c}{Horizon} & \multicolumn{2}{c}{96} & \multicolumn{2}{c}{192} & \multicolumn{2}{c}{336} & \multicolumn{2}{c}{720} \\
    \midrule
    \multicolumn{1}{c}{\(\mu\)} & MSE   & MAE   & MSE   & MAE   & MSE   & MAE   & MSE   & MAE \\
    \midrule
    1     & 0.192 & 0.287 & 0.255 & 0.332 & 0.294 & 0.354 & 0.383 & 0.409 \\
    3     & 0.172 & 0.264 & 0.228 & 0.304 & 0.277 & \textbf{0.336} & \textbf{0.371} & \textbf{0.403} \\
    5     & 0.168 & 0.259 & 0.225 & 0.302 & \textbf{0.275} & 0.337 & 0.389 & 0.420 \\
    7     & 0.166 & \textbf{0.257} & \textbf{0.223} & \textbf{0.300} & 0.279 & 0.343 & 0.440 & 0.451 \\
    9     & \textbf{0.165} & 0.258 & \textbf{0.223} & 0.301 & 0.285 & 0.350 & 0.409 & 0.434 \\
    \bottomrule
    \end{tabular}%
  }
  \label{tab:ll}%
\end{table}%
\begin{table}[t]
  \centering
  \caption{Ablation study on backbone models. {\shortname} refers to our proposed approach, an INR with random Fourier features sampled from a range of scales. MLP refers to replacing the random Fourier features with a linear map from input dimension to hidden dimension. SIREN refers to an INR with periodic activations as proposed by \citet{sitzmann2020implicit}. RNN refers to an autoregressive recurrent neural network (inputs are the time-series values, \(\vy_t\)). All approaches include differentiable closed-form ridge regressor.
  Further model details can be found in \cref{app:ablation-model}.
  }
  \resizebox{\columnwidth}{!}{
    \begin{tabular}{c|c|cccccccc}
    \toprule
    \multicolumn{2}{c}{Methods} & \multicolumn{2}{c}{{\shortname}} & \multicolumn{2}{c}{MLP} & \multicolumn{2}{c}{SIREN} & \multicolumn{2}{c}{RNN} \\
    \midrule
    \multicolumn{2}{c}{Metrics} & MSE   & MAE   & MSE   & MAE   & MSE   & MAE   & MSE   & MAE \\
    \midrule
    \multirow{4}[2]{*}{\begin{sideways}ETTm2\end{sideways}} & 96    & \textbf{0.166} & \textbf{0.257} & 0.186 & 0.284 & 0.236 & 0.325 & 0.233 & 0.324 \\
          & 192   & \textbf{0.225} & \textbf{0.302} & 0.265 & 0.338 & 0.295 & 0.361 & 0.275 & 0.337 \\
          & 336   & \textbf{0.277} & \textbf{0.336} & 0.316 & 0.372 & 0.327 & 0.386 & 0.344 & 0.383 \\
          & 720   & \textbf{0.383} & \textbf{0.409} & 0.401 & 0.417 & 0.438 & 0.453 & 0.431 & 0.432 \\
    \bottomrule
    \end{tabular}%
  }
  \label{tab:ablation-model}%
\end{table}%

\begin{table*}[t]
  \centering
  \caption{Ablation study on variants of {\shortname}. 
  Starting from the original version, we add (+) or remove (-) some component from {\shortname}.
  \textit{Datetime} refers to datetime features. 
  \textit{RR} stands for the differentiable closed-form \textbf{r}idge \textbf{r}egressor, removing it refers to replacing this module with a simple linear layer trained via gradient descent across all training samples (i.e. without meta-optimization formulation).
  \textit{Local} refers to training an INR from scratch via gradient descent for each lookback window (RR is \textbf{not} used here, and there is no training phase). 
  + \textit{Finetune} refers to training an INR via gradient descent for each lookback window on top of having a training phase. 
  \textit{Full MAML} refers to performing gradient steps for the inner loop and backpropagating through them for the outer loop as in \cite{finn2017model}.
  Further details on the variants can be found in \cref{app:ablation-variants}.
  }
  \resizebox{\textwidth}{!}{
    \begin{tabular}{c|c|cccccccccccccccccc}
    \toprule
    \multicolumn{2}{c}{\multirow{2}[1]{*}{Methods}} &
      \multicolumn{2}{c}{\multirow{2}[1]{*}{{\shortname}}} &
      \multicolumn{2}{c}{\multirow{2}[1]{*}{+ Datetime}} &
      \multicolumn{2}{c}{\multirow{2}[1]{*}{- RR}} &
      \multicolumn{2}{c}{- RR} &
      \multicolumn{2}{c}{\multirow{2}[1]{*}{+ Local}} &
      \multicolumn{2}{c}{+ Local} &
      \multicolumn{2}{c}{\multirow{2}[1]{*}{+ Finetune}} &
      \multicolumn{2}{c}{+ Finetune} &
      \multicolumn{2}{c}{\multirow{2}[1]{*}{Full MAML}}
      \\
    \multicolumn{2}{c}{} &
      \multicolumn{2}{c}{} &
      \multicolumn{2}{c}{} &
      \multicolumn{2}{c}{} &
      \multicolumn{2}{c}{+ Datetime} &
      \multicolumn{2}{c}{} &
      \multicolumn{2}{c}{+ Datetime} &
      \multicolumn{2}{c}{} &
      \multicolumn{2}{c}{+ Datetime} &
      \multicolumn{2}{c}{}
      \\
    \midrule
    \multicolumn{2}{c}{Metrics} &
      MSE &
      MAE &
      MSE &
      MAE &
      MSE &
      MAE &
      MSE &
      MAE &
      MSE &
      MAE &
      MSE &
      MAE &
      MSE &
      MAE &
      MSE &
      MAE &
      MSE &
      MAE
      \\
    \midrule
    \multirow{4}[2]{*}{\begin{sideways}ETTm2\end{sideways}} &
      96 &
      \textbf{0.166} &
      \textbf{0.257} &
      0.226 &
      0.303 &
      3.072 &
      1.345 &
      3.393 &
      1.400 &
      0.251 &
      0.331 &
      0.250 &
      0.327 &
      3.028 &
      1.328 &
      3.242 &
      1.365 &
      0.235 &
      0.326
      \\
     &
      192 &
      \textbf{0.225} &
      \textbf{0.302} &
      0.309 &
      0.362 &
      3.064 &
      1.343 &
      3.269 &
      1.381 &
      0.322 &
      0.371 &
      0.323 &
      0.366 &
      3.043 &
      1.341 &
      3.385 &
      1.391 &
      0.295 &
      0.361
      \\
     &
      336 &
      \textbf{0.277} &
      \textbf{0.336} &
      0.341 &
      0.381 &
      2.920 &
      1.309 &
      3.442 &
      1.401 &
      0.370 &
      0.412 &
      0.367 &
      0.396 &
      2.950 &
      1.331 &
      3.367 &
      1.387 &
      0.348 &
      0.392
      \\
     &
      720 &
      \textbf{0.383} &
      \textbf{0.409} &
      0.453 &
      0.447 &
      2.773 &
      1.273 &
      3.400 &
      1.399 &
      0.443 &
      0.449 &
      0.455 &
      0.461 &
      2.721 &
      1.253 &
      3.476 &
      1.407 &
      0.491 &
      0.484
      \\
    \bottomrule
    \end{tabular}%
  }
  \label{tab:ablation-variants}%
\end{table*}%

\begin{figure*}[t]
\centering
\begin{subfigure}[b]{\textwidth}
\centering
\includegraphics[width=0.49\columnwidth]{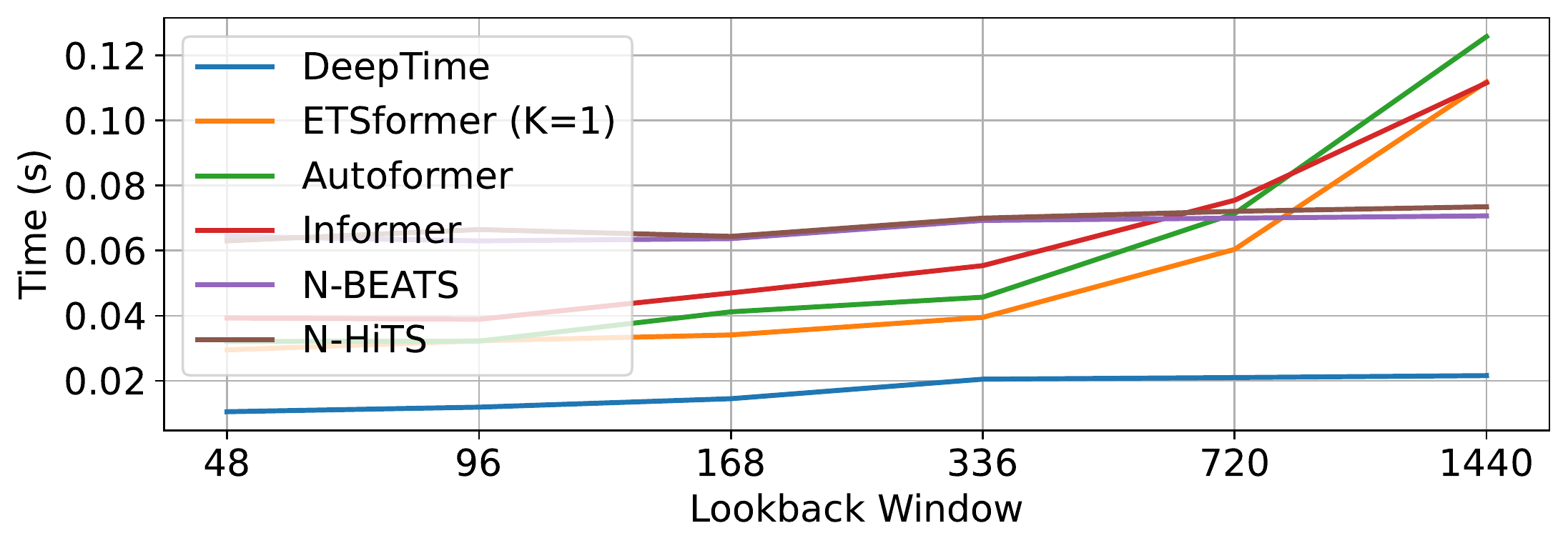}
\includegraphics[width=0.49\columnwidth]{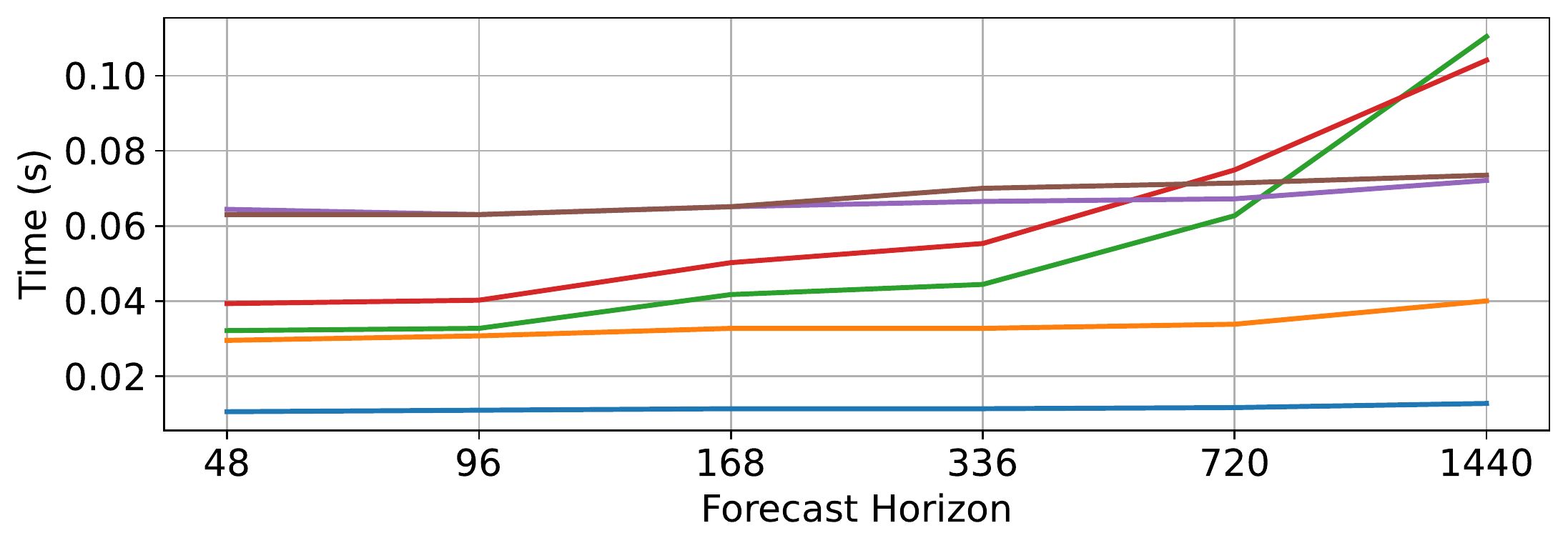}
\caption{Runtime Analysis}
\end{subfigure}
\begin{subfigure}[b]{\textwidth}
\centering
\includegraphics[width=0.49\columnwidth]{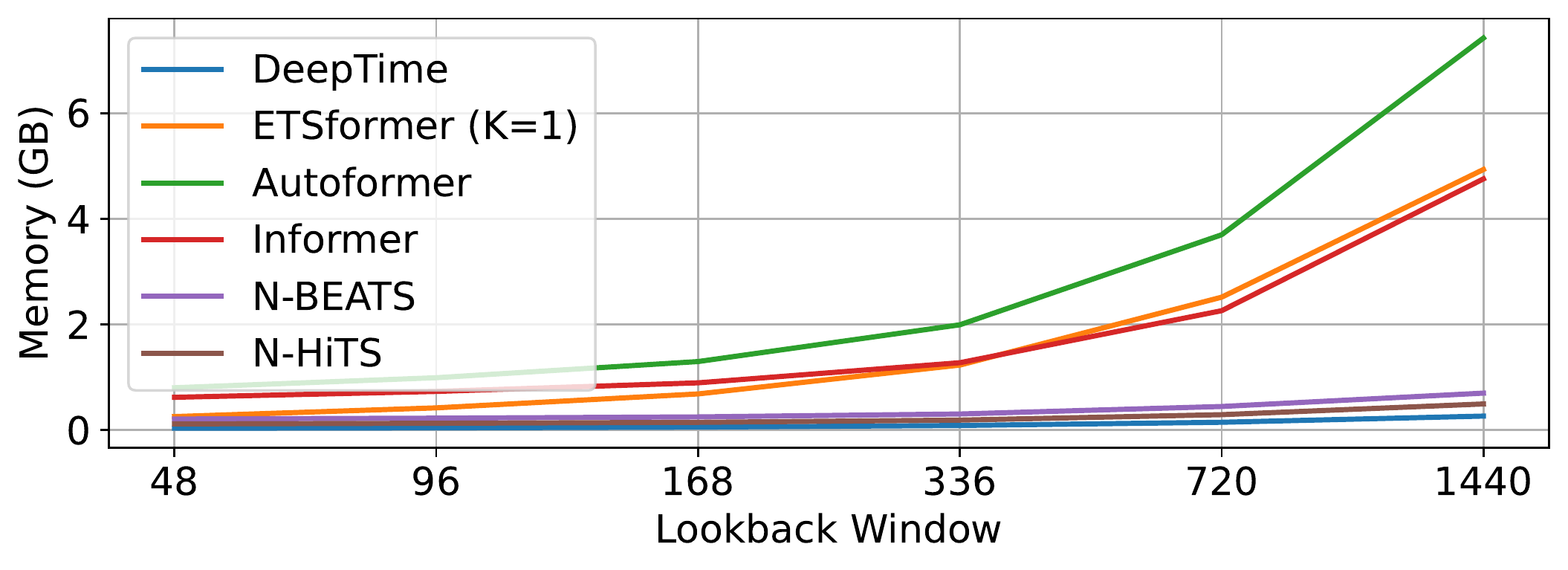}
\includegraphics[width=0.49\columnwidth]{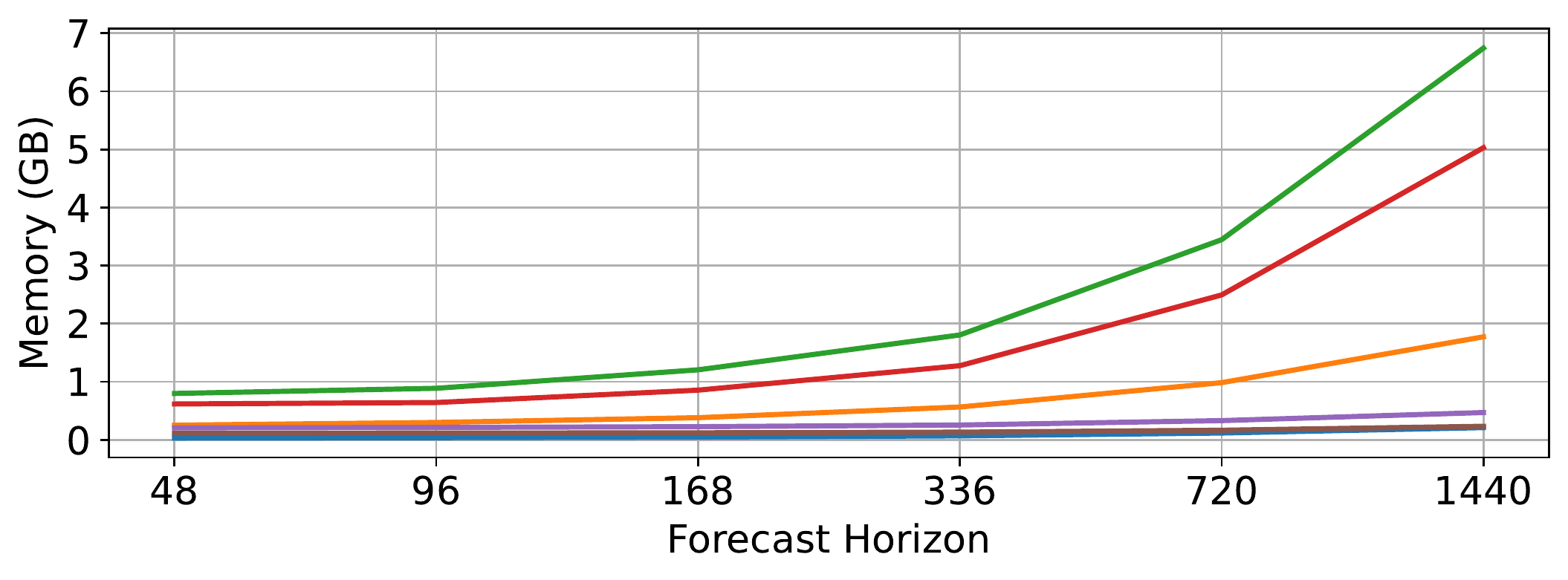}
\caption{Memory Analysis}
\end{subfigure}
\vskip -0.1in
\caption{Computational efficiency benchmark on the ETTm2 multivariate dataset, on a batch size of 32.
Runtime is measured for one iteration (forward + backward pass).
Left: Runtime/Memory usage as lookback length varies, horizon is fixed to 48. 
Right: Runtime/Memory usage as horizon length varies, lookback length is fixed to 48.
Further model details can be found in \cref{app:efficiency}.
}
\vspace{-0.1in}
\label{fig:efficiency}
\end{figure*}

We first perform empirical analyses informed by the insights from our theoretical analysis. 
\cref{thm:bound} states that generalization error is bounded by training error, and two complexity terms. \cref{fig:instances} and \cref{tab:ll} analyse the test error (which approximates generalization error) as number of instances, \(n\), and lookback window length, \(L\), vary.
For the first complexity term, controlled by denominator \(n\), observed in \cref{fig:instances}, test error decreases as \(n\) increases. 
For the second complexity term, controlled by the length of lookback and horizon, since \(H\) is an experimental setting, we set perform a sensitivity analysis on the lookback length, setting \(L = \mu * H\). Similarly, test error decreases as \(L\) increases, plateauing and even increasing slightly as \(L\) grows extremely large. We expect test error to plateau as the associated term goes to zero. As the number of instances available for training decreases as \(L\) grows large, the increase in test error can be attributed to a decrease in \(n\).

In \cref{tab:ablation-model} we perform an ablation study on various backbone architectures, while retaining the differentiable closed-form ridge regressor. We observe a degradation when the random Fourier features layer is removed, due to the spectral bias problem which neural networks face \citep{rahaman2019spectral, tancik2020fourier}. {\shortname} outperforms the SIREN variant of INRs which is consistent with observations INR literature. {\shortname} also outperforms the RNN variant which is the model proposed in \cite{grazzi2021meta}. This is a direct comparison between IMS historical-value models and time-index models, and highlights the benefits of a time-index models.

We perform an ablation study to understand how various training schemes and input features affect the performance of {\shortname}. \cref{tab:ablation-variants} presents these results. First, we observe that our meta-optimization formulation is a critical component to the success of {\shortname}. We note that {\shortname} without meta-optimization may not be a meaningful baseline since the model outputs are always the same regardless of the input lookback window. Including datetime features helps alleviate this issue, yet we observe that the inclusion of datetime features generally lead to a degradation in performance. In the case of {\shortname}, we observed that the inclusion of datetime features lead to a much lower training loss, but degradation in test performance -- this is a case of meta-learning memorization \citep{yin2020meta} due to the tasks becoming non-mutually exclusive \citep{rajendran2020meta}. We also observe that the meta-optimization formulation is indeed superior to training a model from scratch for each lookback window.
Finally, while we expect full MAML to always outperform the fast and efficient meta-optimization, in reality, there are many complications in such gradient-based bi-level optimization methods -- they are difficult to optimize, and instable during training. Restricting the base parameters to only the last layer of the INR provides a useful prior which enables stable optimization and high generalization without facing these problems.

Additional sensitivity analysis on our proposed concatenated Fourier features, can be found in \cref{app:cff}, showing that it performs no worse than an extensive hyperparameter sweep on standard random Fourier features layer.

\subsection{Computational Efficiency}
Finally, we analyse {\shortname}'s efficiency in both runtime and memory usage, with respect to both lookback window and forecast horizon lengths. 
The main bottleneck in computation for {\shortname} is the matrix inversion operation in the ridge regressor, canonically of \(\gO(n^3)\) complexity. This is a major concern for {\shortname} as \(n\) is linked to the length of the lookback window. As mentioned in \cite{bertinetto2018metalearning}, the Woodbury formulation, \[\mW^* = \mZ^T(\mZ\mZ^T + \lambda\mI)^{-1}\mY\] is used to alleviate the problem, leading to an \(\gO(d^3)\) complexity, where \(d\) is the hidden size hyperparameter, fixed to some value (see \cref{app:hyperparams}).
\cref{fig:efficiency} demonstrates that {\shortname} is highly efficient, even when compared to efficient Transformer models, recently proposed for the long sequence time series forecasting task, as well as fully connected models. 
\newpage
\section{Discussion}
In this paper, we proposed {\shortname}, a deep time-index model learned via a meta-optimization framework, to automatically learn a function form from time series data, rather than manually defining the representation function as per classical methods. 
{\shortname} resolves issues arising for vanilla deep time-index models by splitting the learning process into inner and outer learning processes, where the outer learning process enables the deep time-index model to learn a strong inductive bias for extrapolation from data.
We propose a fast and efficient instantiation of the meta-optimization framework, using a closed-form ridge regressor. We also enhance deep time-index models with a novel concatenated Fourier features module to efficiently learn high frequency patterns in time series.
Our extensive empirical analysis shows that {\shortname}, while being a much simpler model architecture compared to prevailing state-of-the-art methods, achieves competitive performance across forecasting benchmarks on real world datasets. We perform substantial ablation studies to identify the key components contributing to the success of {\shortname}, and also show that it is highly efficient.

\vspace{-0.1in}
\paragraph{Limitations \& Future Work} Despite having verified {\shortname}'s effectiveness, we expect some under-performance in cases where the lookback window contains significant anomalies, or an abrupt change point.
Next, while out of scope for our current work, a limitation that {\shortname} faces is that it does not consider holidays and events. We leave the consideration of such features as a potential future direction, along with the incorporation of exogenous covariates and datetime features, whilst avoiding the incursion of the meta-learning memorization problem.
Finally, time-index models are a natural fit for missing value imputation, as well as other time series intelligence tasks for irregular time series -- this is another interesting future direction to extend deep time-index models towards.

% Acknowledgements should only appear in the accepted version.
% \section*{Acknowledgements}

\newpage
\bibliography{icml2023}
\bibliographystyle{icml2023}

%%%%%%%%%%%%%%%%%%%%%%%%%%%%%%%%%%%%%%%%%%%%%%%%%%%%%%%%%%%%%%%%%%%%%%%%%%%%%%%
%%%%%%%%%%%%%%%%%%%%%%%%%%%%%%%%%%%%%%%%%%%%%%%%%%%%%%%%%%%%%%%%%%%%%%%%%%%%%%%
% APPENDIX
%%%%%%%%%%%%%%%%%%%%%%%%%%%%%%%%%%%%%%%%%%%%%%%%%%%%%%%%%%%%%%%%%%%%%%%%%%%%%%%
%%%%%%%%%%%%%%%%%%%%%%%%%%%%%%%%%%%%%%%%%%%%%%%%%%%%%%%%%%%%%%%%%%%%%%%%%%%%%%%
\newpage
\appendix
\onecolumn
\section{{\shortname} Pseudocode}
\label{app:pseudocode}
\begin{algorithm}[H]\small
  \caption{PyTorch-Style Pseudocode of Closed-Form Ridge Regressor}
  \label{alg:ridge}
  \fontsize{8pt}{0em}\selectfont 
  \texttt{mm}: matrix multiplication, \texttt{diagonal}: returns the diagonal elements of a matrix, \texttt{add\_}: in-place addition \\
  \texttt{linalg.solve} computes the solution of a square system of linear equations with a unique solution. 
\begin{lstlisting}[language=python]
# X: inputs, shape: (n_samples, n_dim)
# Y: targets, shape: (n_samples, n_out)
# lambd: scalar value representing the regularization coefficient

n_samples, n_dim = X.shape

# add a bias term by concatenating an all-ones vector
ones = torch.ones(n_samples, 1)
X = cat([X, ones], dim=-1)

if n_samples >= n_dim:
    # standard formulation
    A = mm(X.T, X)
    A.diagonal().add_(softplus(lambd))
    B = mm(X.T, Y)
    weights = linalg.solve(A, B)
else:
    # Woodbury formulation
    A = mm(X, X.T)
    A.diagonal().add_(softplus(lambd))
    weights = mm(X.T, linalg.solve(A, Y))
w, b = weights[:-1], weights[-1:]
return w, b
\end{lstlisting}
\end{algorithm}
\begin{algorithm}[H]\small
  \caption{PyTorch-Style Pseudocode of DeepTIMe}
  \label{alg:deeptime}
  \fontsize{8pt}{0em}\selectfont 
  \texttt{rearrange}: einops style tensor operations \\
  \texttt{mm}: matrix multiplication
\begin{lstlisting}[language=python]
# x: input time-series, shape: (lookback_len, multivariate_dim)
# lookback_len: scalar value representing the length of the lookback window
# horizon_len: scalar value representing the length of the forecast horizon
# inr: implicit neural representation

time_index = linspace(0, 1, lookback_len + horizon_len)  # shape: (lookback_len + horizon_len)
time_index = rearrange(time_index, 't -> t 1')  # shape: (lookback_len + horizon_len, 1)
time_reprs = inr(time_index)  # shape: (lookback_len + horizon_len, hidden_dim)

lookback_reprs = time_reprs[:lookback_len]
horizon_reprs = time_reprs[-horizon_len:]
w, b = ridge_regressor(lookback_reprs, x)
# w.shape = (hidden_dim, multivariate_dim), b.shape = (1, multivariate_dim)
preds = mm(horizon_reprs, w) + b
return preds
\end{lstlisting}
\end{algorithm}
\newpage
\section{Categorization of Forecasting Methods}
\label{app:categorization}
\begin{table}[h]
  \centering
  \caption{Categorization of time-series forecasting methods over the dimensions of time-index vs historical-value methods, and DMS vs IMS methods.}
    \begin{tabular}{|c|c|c|}
    \hline
          & \textbf{Time-index} & \textbf{Historical-value} \bigstrut\\
    \hline
    \multirow{6}[2]{*}{\begin{sideways}\textbf{DMS}\end{sideways}} & {\shortname} & N-HiTS \bigstrut[t]\\
          & Prophet & FEDformer \\
          & Gaussian process & ETSformer \\
          & Time-series regression & Autoformer \\
          &       & Informer \\
          &       & N-BEATS \bigstrut[b]\\
    \hline
    \multirow{4}[2]{*}{\begin{sideways}\textbf{IMS}\end{sideways}} & \multirow{3}[2]{*}{-} & DeepAR \bigstrut[t]\\
          &       & LogTrans \\
          &       & ARIMA \\
          &       & ETS \bigstrut[b]\\
    \hline
    \end{tabular}%
  \label{tab:categorization}%
\end{table}%

\paragraph{Multi-step Forecasts}
Forecasting over a horizon (multiple time steps) can be achieved via two strategies, direct multi-step, or iterative multi-step \citep{marcellino2006comparison, chevillon2007direct, taieb2012recursive}, or even a mixture of both, but this has been less explored:
\begin{itemize}
    \item \textbf{Direct Multi-step (DMS)}: A DMS forecaster directly predicts forecasts for the entire horizon. For example, to achieve a multi-step forecast of \(H\) time steps, a DMS forecaster simply outputs \(H\) values in a single forward pass.
    \item \textbf{Iterative Multi-step (IMS)}: An IMS forecaster iteratively predicts one step ahead, and consumes this forecast to make a subsequent prediction. This is performed iteratively, until the desired length is achieved.
\end{itemize}

\section{Further Discussion on {\shortname} as a Time-index Model}
\label{app:further}
We first reiterate our definitions of time-index and historical-value models from \cref{sec:intro}.
Time-index models are models whose predictions are \textit{purely} functions of \textit{current} time-index features. To perform forecasting (i.e. make predictions over some forecast horizon), time-index models make the predictions \(\hat{\vy}_{t+h} = f(\vtau_{t+h})\) for \(h = 0, \ldots, H - 1\).
Historical-value models predict the time-series value of future time step(s) as a function of past observations, and optionally, covariates.

\begin{minipage}[t]{.5\linewidth}
\center{\textbf{Time-index Models}}
\begin{equation*}
  \hat{\vy}_t = f(\vtau_t)
\end{equation*}
\end{minipage}%
\begin{minipage}[t]{.5\linewidth}
\center{\textbf{Historical-value Models}}
\begin{equation*}
    \hat{\vy}_{t+1} = f(\vy_t, \vy_{t-1}, \ldots, \vz_{t+1}, \vz_{t}, \ldots)
\end{equation*}
\end{minipage}%

Thus, forecasts are of the form, \(\hat{\vy}_{t+h} = \gA(f, \mY_{\lb})(\vtau_{t+h})\), and as can be seen, while the inner loop optimization step is a function of past observations, the adapted time-index model it yields is purely a function of time-index features. 

Next, we further discuss some subtleties of how time-index models interact with past observations. 
Some confusion regarding {\shortname}'s categorization as a time-index model may arise from the above simplified equation for predictions, since forecasts are now a function the lookback window due to the closed-form solution of \(\mW_t^{(K)*}\).
In particular, that \cref{eq:inr,eq:diff-closed-form} indicate that forecasts from {\shortname} are in fact linear in the lookback window. However, we highlight that this is not in contradiction with our definition of historical-value and time-index models. 
Here, we differentiate between the \textit{model}, \(f \in \gH\), and the \textit{learning algorithm}, \(\gA\), which is specified in \cref{eq:inner-loop} (the inner loop optimization). The learning algorithm \(\gA: \gH \times \R^{L \times m} \to \gH\) takes as input a model from the hypothesis class \(\gH\) and, past observations, returning a model minimizing the loss function \(\gL\). A time-index model is thus, still only a function of time-index features, while the learning algorithm is a function of past observations, i.e. \(f, f_0 \in \gH, f: \R^c \to \R^m, f = \gA(f_0, \mY_{\lb})\).
{\shortname} as a forecaster, is a \textbf{deep time-index model endowed with a meta-optimization framework}. In order to perform forecasting, it has to perform an inner loop optimization defined by the learning algorithm, as highlighted in \cref{eq:inner-loop}. For the special case where we use the closed-form ridge regressor, the inner loop learning algorithm reduces to a form which is linear in the lookback window. Still, the deep time-index model is only a function of time-index features.
\section{Generalization Bound for our Meta-optimization Framework}
\label{app:pac-bayes}
In this section, we derive a generalization bound for {\shortname} under the PAC-Bayes framework \citep{mcallester1999pac, shalev2014understanding}. Our formulation follows \citet{amit2018meta} which introduces a meta-learning generalization bound. 
We assume that all instances share the same  hypothesis space $\mathcal{H}$, sample space $\mathcal{Z}$ and loss function $\ell: \mathcal{H} \times \mathcal{Z} \rightarrow [0, 1]$. We observes $n$ instances in the form of sample sets $\gS_1,\dots, \gS_n$. The number of samples in each instance is $H+L$. Each instance $\gS_k$ is assumed to be generated \textit{i.i.d} from an unknown sample distribution $\gD^{H+L}_k$. Each instance's sample distribution $\gD_k$ is \textit{i.i.d.} generated from an unknown meta distribution, $E$. Particularly, we have $\gS_k = (z_{k-L}, \ldots, z_k, \ldots, z_{k+H-1})$, where $z_t=(\vtau_t, \vy_t)$. Here, $\vtau_t$ is the time coordinate, and $\vy_t$ is the time-series value. For any forecaster $h(\cdot)$ parameterized by $\theta$, we define the loss function $\ell(h_\theta, z_t)$. We also define $P$ as the prior distribution over $\mathcal{H}$ and $Q$ as the posterior over $\mathcal{H}$ for each instance. In the meta-learning setting, we assume a hyper-prior $\mathcal{P}$, which is a prior distribution over priors, observes a sequence of training instances, and then outputs a distribution over priors, called hyper-posterior $\mathcal{Q}$. 
We restate \cref{thm:bound} in the following:

\begin{theorem}{(Generalization Bound)}
Let \(\mathcal{Q}, Q\) be  arbitrary distribution of \(\phi, \theta\), which are defined in \cref{eq:outer-loop} and \cref{eq:inner-loop}, and \(\mathcal{P}, P\) be the prior distribution of \(\phi, \theta\). Then for any $c_1,c_2>0$ and any $\delta\in(0,1]$, with probability at least $1-\delta$, the following inequality holds uniformly for all hyper-posterior distributions $\gQ$,
\begin{align}
er(\mathcal{Q}) \leq \frac{c_1c_2}{(1-e^{-c_1})(1-e^{-c_2})}\cdot \frac{1}{n}\sum_{k=1}^{n}\hat{er}(\mathcal{Q}, \gS_k) \nonumber\\
+ \frac{c_1}{1-e^{-c_1}} \cdot  \frac{\mathrm{KL}(\gQ||\gP)+\log \frac{2}{\delta}}{nc_1} \nonumber\\
+ \frac{c_1c_2}{(1-e^{-c_2})(1-e^{-c_1})} \cdot \frac{\mathrm{KL}(Q||P)+\log \frac{2n}{\delta}}{(H+L)c_2}
\end{align}
where \(er(\gQ)\) and \(\hat{er}(\gQ, \gS_k)\) are the generalization error and training error of {\shortname}, respectively.
\end{theorem}

\begin{proof}
Our proof contains two steps. First, we bound the error within observed instances due to observing a limited number of samples. Then we bound the error on the instance environment level due to observing a finite number of instances. Both of the two steps utilize Catoni's classical PAC-Bayes bound \citep{catoni2007pac} to measure the error. Here, we give Catoni's classical PAC-Bayes bound.

\begin{theorem}{(Catoni's bound \citep{catoni2007pac})} 
\label{thm:catoni} Let $\mathcal{X}$ be a sample space, $P(X)$ a distribution over $\mathcal{X}$, $\Theta$ a hypothesis space. Given a loss function $\ell(\theta, X): \Theta \times \mathcal{X} \rightarrow [0, 1]$ and a collection of M i.i.d random variables ($X_1, \ldots, X_M$) sampled from $P(X)$. Let $\pi$ be a prior distribution over hypothesis space. Then, for any $\delta \in (0,1]$ and any real number $c>0$, the following bound holds uniformly for all posterior distributions $\rho$ over hypothesis space,
\begin{align}
P\left(
\underset{X_i\sim P(X), \theta\sim\rho}{\mathbb{E}}[\ell(\theta, X_i)] \leq
\frac{c}{1-e^{-c}}\Big[\frac{1}{M}\sum_{m=1}^{M}\underset{\theta\sim\rho}{\mathbb{E}}[\ell(\theta, X_m)]
+ \frac{\mathrm{KL}(\rho||\pi)+\log \frac{1}{\delta}}{Mc}\Big], \forall\rho 
\right) \nonumber \\
\geq 1 -\delta. \nonumber
\end{align}
\end{theorem}
We first utilize \cref{thm:catoni} to bound the generalization error in each of the observed instances. Let $k$ be the index of instance, we have the definition of expected error and empirical error as follows,
\begin{align}
er(\mathcal{Q}, \gD_k)
& =\underset{P\sim \mathcal{Q}}{\mathbb{E}}
\quad\underset{h\sim Q(\gS_k, P)}{\mathbb{E}}
\quad\underset{z\sim \gD_k}{\mathbb{E}} \ell(h, z), \\
\hat{er}(\mathcal{Q}, \gS_k) & =\underset{P\sim \mathcal{Q}}{\mathbb{E}}
\quad\underset{h\sim Q(\gS_k, P)}{\mathbb{E}}
\quad \frac{1}{H+L}\sum_{j=k-L}^{k+H-1} \ell(h, z_j).
\end{align}

Then, according to \cref{thm:catoni}, for any $\delta_k \sim (0, 1]$ and $c_2>0$, we have
\begin{align}
P\left(
er(\gQ, \gD_k) \leq \frac{c_2}{1-e^{-c_2}}\hat{er}(\mathcal{Q}, \gS_k)
+ \frac{c_2}{1-e^{-c_2}} \cdot  \frac{\mathrm{KL}(Q||P)+\log \frac{1}{\delta_k}}{(H+L)c_2}
\right) \geq 1 - \delta_k. \label{eq:in_bound}
\end{align}

Next, we bound the error due to observing a limited number of instances from the environment. Similarly, we have the definition of expected instance error as follows
\begin{align}
er(\mathcal{Q}) 
&= \underset{D \sim E}{\mathbb{E}}
\quad\underset{\gS\sim \gD^{H+L}}{\mathbb{E}}
\quad\underset{P\sim \mathcal{Q}}{\mathbb{E}}
\quad\underset{h\sim Q(\gS, P)}{\mathbb{E}}
\quad\underset{z\sim D}{\mathbb{E}} \ell(h, z) \nonumber \\
&=\underset{D \sim E}{\mathbb{E}}
\quad\underset{\gS\sim \gD^{H+L}}{\mathbb{E}}er(\mathcal{Q}, D).
\end{align}
Then we have the definition of error across the $n$ instances,
\begin{align}
\frac{1}{n}\sum_{k=1}^{n}\underset{P \sim \mathcal{Q}}{\mathbb{E}}
\quad\underset{h\sim Q(\gS_k, P)}{\mathbb{E}}
\quad\underset{z\sim \gD_k}{\mathbb{E}} \ell(h, z) 
=\frac{1}{n}\sum_{k=1}^{n} er(\mathcal{Q}, \gD_k).
\end{align}
Then \cref{thm:catoni} says that the following holds for any $\delta_0 \sim (0, 1]$ and $c_1>0$, we have
\begin{align}
P\left(
er(\mathcal{Q}) \leq \frac{c_1}{1-e^{-c_1}}\frac{1}{n}\sum_{k=1}^{n} er(\mathcal{Q}, \gD_k)
+ \frac{c_1}{1-e^{-c_1}} \cdot  \frac{\mathrm{KL}(\mathcal{Q}||\mathcal{P})+\log \frac{1}{\delta_0}}{nc_1}
\right) \geq 1 - \delta_0. \label{eq:between_bound}
\end{align}
Finally, by employing a union bound argument (Lemma 1, \citet{amit2018meta}), we could bound the probability of the intersection of the events in \cref{eq:between_bound} and \cref{eq:in_bound} For any $\delta>0$, set $\delta_0 = \frac{\delta}{2}$ and $\delta_k = \frac{\delta}{2n}$ for $k=1,\dots,n$, 
\begin{align}
P\left(
er(\mathcal{Q}) \leq \frac{c_1c_2}{(1-e^{-c_1})(1-e^{-c_2})}\cdot \frac{1}{n}\sum_{k=1}^{n}\hat{er}(\mathcal{Q}, \gS_k)
+ \frac{c_1}{1-e^{-c_1}} \cdot  \frac{\mathrm{KL}(\mathcal{Q}||\mathcal{P})+\log \frac{2}{\delta}}{nc_1}
\right. \nonumber\\
\left. 
+ \frac{c_1c_2}{(1-e^{-c_2})(1-e^{-c_1})} \cdot \frac{\mathrm{KL}(Q||P)+\log \frac{2n}{\delta}}{(H+L)c_2}
\right) \geq 1 - \delta.
\end{align}
\end{proof}

% \cref{thm:ml} shows that the expected instance generalization error is bounded by the empirical multi-instance error plus two complexity terms.
% The first term represents the complexity of the environment, or equivalently, the time-series dataset, converging to zero if we observe an infinitely long time-series ($n \rightarrow \infty$).
% The second term represents the complexity of the observed instances, or equivalently, the lookback-horizon windows. This converges to zero when there are sufficient number of time steps in each window ($H + L \rightarrow \infty$).
\newpage
\section{Synthetic Data}
\label{app:synthetic}
The training set for each synthetic data experiment consists 1000 functions/tasks, while the test set contains 100 functions/tasks. We ensure that there is no overlap between the train and test sets.

\paragraph{Linear}
Samples are generated from the function \(y = ax + b\) for \(x \in [-1, 1]\). This means that each function/task consists of 400 evenly spaced points between -1 and 1. The parameters of each function/task (i.e. \(a, b\)) are sampled from a normal distribution with mean 0 and standard deviation of 50, i.e. \(a, b \sim \gN(0, 50^2)\).

\paragraph{Cubic}
Samples are generated from the function \(y = ax^3 + bx^2 + cx +d\) for \(x \in [-1, 1]\) for 400 points. Parameters of each task are sampled from a continuous uniform distribution with minimum value of -50 and maximum value of 50, i.e. \(a, b, c, d \sim \gU(-50, 50)\).

\paragraph{Sums of sinusoids}
Sinusoids come from a fixed set of frequencies, generated by sampling \(\omega \sim \gU(0, 12 \pi)\). We fix the size of this set to be five, i.e. \(\Omega = \{\omega_1, \ldots, \omega_5\}\). 
Each function is then a sum of \(J\) sinusoids, where \(J \in \{1, 2, 3, 4, 5\}\) is randomly assigned. 
The function is thus \(y = \sum_{j=1}^J A_j \sin(\omega_{r_j} x + \varphi_j)\) for \(x \in [0, 1]\), where the amplitude and phase shifts are freely chosen via \(A_j \sim \gU(0.1, 5), \varphi_j \sim \gU(0, \pi)\), but the frequency is decided by \(r_j \in \{1, 2, 3, 4, 5\}\) to randomly select a frequency from the set \(\Omega\).

The predictions from {\shortname} in \cref{fig:synthetic} demonstrate some noise, likely stemming from the model's capability to learn high frequency features due to the use of implicit neural representations with random Fourier features. Since the synthetic data are all low frequency, smoothly changing functions, the noise is likely to be artifacts from the concatenated Fourier features layer, which should go away if the scale parameter of the Fourier features are carefully fine-tuned. However, the power of our proposed concatenated Fourier features layer is that the model is able to fit to both high and low frequency features without tuning, though at the expense of some noise as seen in the figure.
\section{Datasets}
\label{app:datasets}

\textbf{ETT}\footnote{\url{https://github.com/zhouhaoyi/ETDataset}} \cite{zhou2021informer} - Electricity Transformer Temperature provides measurements from an electricity transformer such as load and oil temperature. We use the \textit{ETTm2} subset, consisting measurements at a 15 minutes frequency.

\textbf{ECL}\footnote{\url{https://archive.ics.uci.edu/ml/datasets/ElectricityLoadDiagrams20112014}} - Electricity Consuming Load provides measurements of electricity consumption for 321 households from 2012 to 2014. The data was collected at the 15 mintue level, but is aggregated hourly.

\textbf{Exchange}\footnote{\url{https://github.com/laiguokun/multivariate-time-series-data}} \cite{lai2018modeling} - a collection of daily exchange rates with USD of eight countries (Australia, United Kingdom, Canada, Switzerland, China, Japan, New Zealand, and Singapore) from 1990 to 2016.

\textbf{Traffic}\footnote{\url{https://pems.dot.ca.gov/}} - dataset from the California Department of Transportation providing the hourly road occupancy rates from 862 sensors in San Francisco Bay area freeways.

\textbf{Weather}\footnote{\url{https://www.bgc-jena.mpg.de/wetter/}} - provides measurements of 21 meteorological indicators such as air temperature, humidity, etc., every 10 minutes for the year of 2020 from the Weather Station of the Max Planck Biogeochemistry Institute in Jena, Germany.

\textbf{ILI}\footnote{\url{https://gis.cdc.gov/grasp/fluview/fluportaldashboard.html}} - Influenza-like Illness measures the weekly ratio of patients seen with ILI and the total number of patients, obtained by the Centers for Disease Control and Prevention of the United States between 2002 and 2021.

\newpage

\newpage
\section{{\shortname} Implementation Details}
\label{app:implementation}
\paragraph{Optimization}
We train {\shortname} with the Adam optimizer \citep{kingma2014adam} with a learning rate scheduler following a linear warm up and cosine annealing scheme. Gradient clipping by norm is applied. The ridge regressor regularization coefficient, \(\lambda\), is trained with a different, higher learning rate than the rest of the meta parameters. We use early stopping based on the validation loss, with a fixed patience hyperparameter (number of epochs for which loss deteriorates before stopping). All experiments are performed on an Nvidia A100 GPU.

\paragraph{Model}
The ridge regression regularization coefficient is a learnable parameter constrained to positive values via a softplus function. We apply Dropout \citep{srivastava2014dropout}, then LayerNorm \citep{ba2016layer} after the ReLU activation function in each INR layer. The size of the random Fourier feature layer is set independently of the layer size, in which we define the total size of the random Fourier feature layer -- the number of dimensions for each scale is divided equally.

\section{{\shortname} Hyperparameters}
\label{app:hyperparams}
% Table generated by Excel2LaTeX from sheet 'Hyperparams'
\begin{table}[H]
  \centering
  \caption{Hyperparameters used in {\shortname}.}
    \begin{tabular}{cll}
    \toprule
          & Hyperparameter & Value \\
    \midrule
    \multirow{7}[2]{*}{\begin{sideways}Optimization\end{sideways}} & Epochs & 50 \\
          & Learning rate & 1e-3 \\
          & \(\lambda\) learning rate & 1.0 \\
          & Warm up epochs & 5 \\
          & Batch size & 256 \\
          & Early stopping patience & 7 \\
          & Max gradient norm & 10.0 \\
    \midrule
    \multirow{7}[2]{*}{\begin{sideways}Model\end{sideways}} & Layers & 5 \\
          & Layer size & 256 \\
          & \(\lambda\) initialization & 0.0 \\
          & Scales & \([0.01, 0.1, 1, 5, 10, 20, 50, 100]\) \\
          & Fourier features size & 4096 \\
          & Dropout & 0.1 \\
          & Lookback length multiplier, \(\mu\) & \(\mu \in \{1,3,5,7,9\}\) \\
    \bottomrule
    \end{tabular}%
  \label{tab:hyperparams}%
\end{table}%

\section{Univariate Forecasting Benchmark}
\label{app:univar-benchmark}
\begin{table}[H]
  \centering
  \caption{Univariate forecasting benchmark on long sequence time-series forecasting. Best results are highlighted in \textbf{bold}, and second best results are \underline{underlined}.}
  \resizebox{\textwidth}{!}{
\begin{tabular}{c|c|ccccccrrrrrrrrrrrrrrcc}
\toprule
\multicolumn{2}{c}{Methods} & \multicolumn{2}{c}{{\shortname}} & \multicolumn{2}{c}{N-HiTS} & \multicolumn{2}{c}{ETSformer} & \multicolumn{2}{c}{Fedformer} & \multicolumn{2}{c}{Autoformer} & \multicolumn{2}{c}{Informer} & \multicolumn{2}{c}{N-BEATS} & \multicolumn{2}{c}{DeepAR} & \multicolumn{2}{c}{Prophet} & \multicolumn{2}{c}{ARIMA} & \multicolumn{2}{c}{GP} \\
\midrule
\multicolumn{2}{c}{Metrics} & MSE   & MAE   & MSE   & MAE   & MSE   & MAE   & \multicolumn{1}{c}{MSE} & \multicolumn{1}{c}{MAE} & \multicolumn{1}{c}{MSE} & \multicolumn{1}{c}{MAE} & \multicolumn{1}{c}{MSE} & \multicolumn{1}{c}{MAE} & \multicolumn{1}{c}{MSE} & \multicolumn{1}{c}{MAE} & \multicolumn{1}{c}{MSE} & \multicolumn{1}{c}{MAE} & \multicolumn{1}{c}{MSE} & \multicolumn{1}{c}{MAE} & \multicolumn{1}{c}{MSE} & \multicolumn{1}{c}{MAE} & MSE   & MAE \\
\midrule
\multirow{4}[2]{*}{\begin{sideways}ETTm2\end{sideways}} & 96    & \underline{0.065} & \underline{0.186} & 0.066 & \textbf{0.185} & 0.080 & 0.212 & \textbf{0.063} & 0.189 & 0.065 & 0.189 & 0.088 & 0.225 & 0.082 & 0.219 & 0.099 & 0.237 & 0.287 & 0.456 & 0.211 & 0.362 & 0.125 & 0.273 \\
      & 192   & \underline{0.096} & \underline{0.234} & \textbf{0.087} & \textbf{0.223} & 0.150 & 0.302 & 0.102 & 0.245 & 0.118 & 0.256 & 0.132 & 0.283 & 0.120 & 0.268 & 0.154 & 0.310 & 0.312 & 0.483 & 0.261 & 0.406 & 0.154 & 0.307 \\
      & 336   & 0.138 & 0.285 & \textbf{0.106} & \textbf{0.251} & 0.175 & 0.334 & \underline{0.130} & \underline{0.279} & 0.154 & 0.305 & 0.180 & 0.336 & 0.226 & 0.370 & 0.277 & 0.428 & 0.331 & 0.474 & 0.317 & 0.448 & 0.189 & 0.338 \\
      & 720   & 0.186 & 0.338 & \textbf{0.157} & \textbf{0.312} & 0.224 & 0.379 & \underline{0.178} & \underline{0.325} & 0.182 & 0.335 & 0.300 & 0.435 & 0.188 & 0.338 & 0.332 & 0.468 & 0.534 & 0.593 & 0.366 & 0.487 & 0.318 & 0.421 \\
\midrule
\multirow{4}[2]{*}{\begin{sideways}Exchange\end{sideways}} & 96    & \textbf{0.086} & \underline{0.226} & \underline{0.093} & \textbf{0.223} & 0.099 & 0.230 & 0.131 & 0.284 & 0.241 & 0.299 & 0.591 & 0.615 & 0.156 & 0.299 & 0.417 & 0.515 & 0.828 & 0.762 & 0.112 & 0.245 & 0.165 & 0.311 \\
      & 192   & \textbf{0.173} & \underline{0.330} & 0.230 & \textbf{0.313} & \underline{0.223} & 0.353 & 0.277 & 0.420 & 0.273 & 0.665 & 1.183 & 0.912 & 0.669 & 0.665 & 0.813 & 0.735 & 0.909 & 0.974 & 0.304 & 0.404 & 0.649 & 0.617 \\
      & 336   & 0.539 & 0.575 & \textbf{0.370} & \textbf{0.486} & \underline{0.421} & \underline{0.497} & 0.426 & 0.511 & 0.508 & 0.605 & 1.367 & 0.984 & 0.611 & 0.605 & 1.331 & 0.962 & 1.304 & 0.988 & 0.736 & 0.598 & 0.596 & 0.592 \\
      & 720   & \underline{0.936} & \underline{0.763} & \textbf{0.728} & \textbf{0.569} & 1.114 & 0.807 & 1.162 & 0.832 & 0.991 & 0.860 & 1.872 & 1.072 & 1.111 & 0.860 & 1.890 & 1.181 & 3.238 & 1.566 & 1.871 & 0.935 & 1.002 & 0.786 \\
\bottomrule
\end{tabular}%
  }
  \label{tab:main-uni}%
\end{table}%

\newpage
\section{{\shortname} Standard Deviation}
\label{app:sd}
\begin{table}[H]
    \caption{{\shortname} main benchmark results with standard deviation. Experiments are performed over three runs.}
    \begin{subtable}[t]{.5\linewidth}
      \centering
        \caption{Multivariate benchmark.}
        \begin{tabular}{c|c|ll}
        \toprule
        \multicolumn{2}{c}{Metrics} & MSE (SD) & MAE (SD) \\
        \midrule
        \multirow{4}[2]{*}{\begin{sideways}ETTm2\end{sideways}} & 96    & 0.166 (0.000) & 0.257 (0.001) \\
              & 192   & 0.225 (0.001) & 0.302 (0.003) \\
              & 336   & 0.277 (0.002) & 0.336 (0.002) \\
              & 720   & 0.383 (0.007) & 0.409 (0.006) \\
        \midrule
        \multirow{4}[2]{*}{\begin{sideways}ECL\end{sideways}} & 96    & 0.137 (0.000) & 0.238 (0.000) \\
              & 192   & 0.152 (0.000) & 0.252 (0.000) \\
              & 336   & 0.166 (0.000) & 0.268 (0.000) \\
              & 720   & 0.201 (0.000) & 0.302 (0.000) \\
        \midrule
        \multirow{4}[2]{*}{\begin{sideways}Exchange\end{sideways}} & 96    & 0.081 (0.001) & 0.205 (0.002) \\
              & 192   & 0.151 (0.002) & 0.284 (0.003) \\
              & 336   & 0.314 (0.033) & 0.412 (0.020) \\
              & 720   & 0.856 (0.202) & 0.663 (0.082) \\
        \midrule
        \multirow{4}[2]{*}{\begin{sideways}Traffic\end{sideways}} & 96    & 0.390 (0.001) & 0.275 (0.001) \\
              & 192   & 0.402 (0.000) & 0.278 (0.000) \\
              & 336   & 0.415 (0.000) & 0.288 (0.001) \\
              & 720   & 0.449 (0.000) & 0.307 (0.000) \\
        \midrule
        \multirow{4}[2]{*}{\begin{sideways}Weather\end{sideways}} & 96    & 0.166 (0.001) & 0.221 (0.002) \\
              & 192   & 0.207 (0.000) & 0.261 (0.000) \\
              & 336   & 0.251 (0.000) & 0.298 (0.001) \\
              & 720   & 0.301 (0.001) & 0.338 (0.001) \\
        \midrule
        \multirow{4}[2]{*}{\begin{sideways}ILI\end{sideways}} & 24    & 2.425 (0.058) & 1.086 (0.027) \\
              & 36    & 2.231 (0.087) & 1.008 (0.011) \\
              & 48    & 2.230 (0.144) & 1.016 (0.037) \\
              & 60    & 2.143 (0.032) & 0.985 (0.016) \\
        \bottomrule
        \end{tabular}%
    \end{subtable}%
    \begin{subtable}[t]{.5\linewidth}
      \centering
        \caption{Univariate benchmark.}
        \begin{tabular}{c|c|ll}
        \toprule
        \multicolumn{2}{c}{Metrics} & MSE (SD) & MAE (SD) \\
        \midrule
        \multirow{4}[2]{*}{\begin{sideways}ETTm2\end{sideways}} & 96    & 0.065 (0.000) & 0.186 (0.000) \\
              & 192   & 0.096 (0.002) & 0.234 (0.003) \\
              & 336   & 0.138 (0.001) & 0.285 (0.001) \\
              & 720   & 0.186 (0.002) & 0.338 (0.002) \\
        \midrule
        \multirow{4}[2]{*}{\begin{sideways}Exchange\end{sideways}} & 96    & 0.086 (0.000) & 0.226 (0.000) \\
              & 192   & 0.173 (0.004) & 0.330 (0.003) \\
              & 336   & 0.539 (0.066) & 0.575 (0.027) \\
              & 720   & 0.936 (0.222) & 0.763 (0.075) \\
        \bottomrule
        \end{tabular}%
    \end{subtable} 
    \label{tab:sd}
\end{table}

\newpage
\section{Random Fourier Features Scale Hyperparameter Sensitivity Analysis}
\label{app:cff}
% Table generated by Excel2LaTeX from sheet 'CFF'
\begin{table*}[h]
  \centering
  \caption{Comparison of CFF against the optimal and pessimal scales as obtained from the hyperparameter sweep. We also calculate the change in performance between CFF and the optimal and pessimal scales, where a positive percentage refers to a CFF underperforming, and negative percentage refers to CFF outperforming, calculated as \(\% \; \mathrm{change} = (\mathrm{MSE}_{CFF} - \mathrm{MSE}_{Scale})/\mathrm{MSE}_{Scale}\).}
  \small
    \begin{tabular}{c|c|cccccc}
    \toprule
    \multicolumn{2}{c}{} & \multicolumn{2}{c}{CFF} & \multicolumn{2}{c}{Optimal Scale (\% change)} & \multicolumn{2}{c}{Pessimal Scale (\% change)} \\
    \midrule
    \multicolumn{2}{c}{Metrics} & MSE   & MAE   & MSE   & MAE   & MSE   & MAE \\
    \midrule
    \multirow{4}[2]{*}{\begin{sideways}ETTm2\end{sideways}} & 96    & 0.166 & 0.257 & 0.164 (1.20\%) & 0.257 (-0.05\%) & 0.216 (-23.22\%) & 0.300 (-14.22\%) \\
          & 192   & 0.225 & 0.302 & 0.220 (1.87\%) & 0.301 (0.25\%) & 0.275 (-18.36\%) & 0.340 (-11.25\%) \\
          & 336   & 0.277 & 0.336 & 0.275 (0.70\%) & 0.336 (-0.22\%) & 0.340 (-18.68\%) & 0.375 (-10.57\%) \\
          & 720   & 0.383 & 0.409 & 0.364 (5.29\%) & 0.392 (4.48\%) & 0.424 (-9.67\%) & 0.430 (-4.95\%) \\
    \bottomrule
    \end{tabular}%
  \label{tab:cff}%
\end{table*}%

% Table generated by Excel2LaTeX from sheet 'CFF'
\begin{table}[H]
  \centering
  \caption{Results from hyperparameter sweep on the scale hyperparameter. Best scores are highlighted in \textbf{bold}, and worst scores are highlighted in \textcolor{red}{\textbf{bold red}}.}
  \resizebox{\textwidth}{!}{
    \begin{tabular}{c|c|cccccccccccccccc}
    \toprule
    \multicolumn{2}{c}{Scale Hyperparam} & \multicolumn{2}{c}{0.01} & \multicolumn{2}{c}{0.1} & \multicolumn{2}{c}{1} & \multicolumn{2}{c}{5} & \multicolumn{2}{c}{10} & \multicolumn{2}{c}{20} & \multicolumn{2}{c}{50} & \multicolumn{2}{c}{100} \\
    \midrule
    \multicolumn{2}{c}{Metrics} & MSE   & MAE   & MSE   & MAE   & MSE   & MAE   & MSE   & MAE   & MSE   & MAE   & MSE   & MAE   & MSE   & MAE   & MSE   & MAE \\
    \midrule
    \multirow{4}[2]{*}{\begin{sideways}ETTm2\end{sideways}} & 96    & \textcolor[rgb]{ 1,  0,  0}{\textbf{0.216}} & \textcolor[rgb]{ 1,  0,  0}{\textbf{0.300}} & 0.189 & 0.285 & 0.173 & 0.268 & 0.168 & 0.262 & 0.166 & 0.260 & 0.165 & 0.258 & 0.165 & 0.259 & \textbf{0.164} & \textbf{0.257} \\
          & 192   & \textcolor[rgb]{ 1,  0,  0}{\textbf{0.275}} & \textcolor[rgb]{ 1,  0,  0}{\textbf{0.340}} & 0.264 & 0.333 & 0.239 & 0.317 & 0.225 & \textbf{0.301} & 0.225 & 0.303 & 0.224 & 0.302 & 0.224 & 0.304 & \textbf{0.220} & 0.301 \\
          & 336   & \textcolor[rgb]{ 1,  0,  0}{\textbf{0.340}} & \textcolor[rgb]{ 1,  0,  0}{\textbf{0.375}} & 0.319 & 0.371 & 0.292 & 0.351 & \textbf{0.275} & 0.337 & 0.277 & \textbf{0.336} & 0.282 & 0.345 & 0.278 & 0.342 & 0.280 & 0.344 \\
          & 720   & \textcolor[rgb]{ 1,  0,  0}{\textbf{0.424}} & 0.430 & 0.405 & 0.420 & 0.381 & 0.412 & \textbf{0.364} & \textbf{0.392} & 0.375 & 0.408 & 0.410 & \textcolor[rgb]{ 1,  0,  0}{\textbf{0.430}} & 0.396 & 0.423 & 0.406 & 0.429 \\
    \bottomrule
    \end{tabular}%
  }
  \label{tab:scales}%
\end{table}%

We perform a comparison between the optimal and pessimal scale hyperparameter for the vanilla random Fourier features layer, against our proposed CFF. 
We first report the results on each scale hyperparameter for the vanilla random Fourier features layer in \cref{tab:scales}. 
As with the other ablation studies, the results reported in \cref{tab:scales} is based on performing a hyperparameter sweep across lookback length multiplier, and selecting the optimal settings based on the validation set, and reporting the test set results. 
Then, the optimal and pessimal scales are simply the best and worst results based on \cref{tab:scales}. \cref{tab:cff} shows that CFF achieves extremely low deviation from the optimal scale across all settings, yet retains the upside of avoiding this expensive hyperparameter tuning phase. We also observe that tuning the scale hyperparameter is extremely important, as CFF obtains up to a 23.22\% improvement in MSE over the pessimal scale hyperparameter.

\section{Ablation Studies Details}
\label{app:ablation}

In this section, we list more details on the models compared to in the ablation studies section. Unless otherwise stated, we perform the same hyperparameter tuning for all models in the ablation studies, and use the same standard hyperparameters such as number of layers, layer size, etc.

\subsection{Ablation study on variants of {\shortname}}
\label{app:ablation-variants}

\paragraph{Datetime Features}
As each dataset comes with a timestamps for each observation, we are able to construct datetime features from these timestamps. We construct the following features:
\begin{enumerate}
    \item Quarter-of-year
    \item Month-of-year
    \item Week-of-year
    \item Day-of-year
    \item Day-of-month
    \item Day-of-week
    \item Hour-of-day
    \item Minute-of-hour
    \item Second-of-minute
\end{enumerate}
Each feature is initially an integer value, e.g. month-of-year can take on values in \(\{0, 1, \ldots, 11\}\), which we subsequently normalize to a \([0,1]\) range. Depending on the data sampling frequency, the appropriate features can be chosen. For the ETTm2 dataset, we used all features except second-of-minute since it is sampled at a 15 minute frequency.

\paragraph{RR} Removing the ridge regressor module refers to replacing it with a simple linear layer, \(\mathrm{Linear}: \R^d \to \R^m\), where \(\mathrm{Linear}(\vx) = \mW \vx + \vb\), \(\vx \in \R^d, \mW \in \R^{m \times d}, \vb \in \R^m\). This corresponds to a straight forward INR, which is trained across all lookback-horizon pairs in the dataset.

\paragraph{Local} For models marked ``Local'', we similarly remove the ridge regressor module and replace it with a linear layer. Yet, the model is not trained across all lookback-horizon pairs in the dataset. Instead, for each lookback-horizon pair in the validation/test set, we fit the model to the lookback window via gradient descent, and then perform prediction on the horizon to obtain the forecasts. A new model is trained from scratch for each lookback-horizon window. We perform tuning on an extra hyperparameter, the number of epochs to perform gradient descent, for which we search through \(\{10, 20, 30, 40, 50\}\).

\paragraph{Finetune} Models marked ``Finetune'' are similar to ``Local'', except that they have been trained on the training set first, and for each lookback-horizon pair in the test set, they are ``finetuned'' on the lookback window.

\paragraph{Full MAML} ``Full MAML'' indicates the setting for which MAML is performed on the entire deep time-index model, by backpropagating through inner loop gradient steps as per \citet{finn2017model}, rather than our proposed fast and efficient meta-optimization framework. Inner loop optimization is performed using the Adam optimizer, and is tuned over lookback length multiplier values of 
\(\{1, 3, 5, 7, 9\}\), and inner loop iterations of \(\{1, 5, 10\}\).

\subsection{Ablation study on backbone models}
\label{app:ablation-model}
For all models in this section, we retain the differentiable closed-form ridge regressor, to identify the effects of the backbone model used.

\paragraph{MLP}
The random Fourier features layer is a mapping from coordinate space to latent space \(\gamma: \R^c \to \R^d\). To remove the effects of the random Fourier features layer, we simply replace it with a with a linear map, \(\mathrm{Linear}: \R^c \to \R^d\). 

\paragraph{SIREN}
We replace the random Fourier features backbone with the SIREN model which is introduced by \cite{sitzmann2020implicit}. In this model, periodical activation functions are used, i.e. \(\sin(\vx)\), along with specified weight initialization scheme.

\paragraph{RNN} We use a 2 layer LSTM with hidden size of 256. Inputs are observations, \(\vy_t\), in an IMS fashion, predicting the next time step, \(\vy_{t+1}\). 

\section{Computational Efficiency Experiments Details}
\label{app:efficiency}

\paragraph{Trans/In/Auto/ETS-former} 
We use a model with 2 encoder and 2 decoder layers with a hidden size of 512, as specified in their original papers.

\paragraph{N-BEATS} We use an N-BEATS model with 3 stacks and 3 layers (relatively small compared to 30 stacks and 4 layers used in their orignal paper\footnote{\url{https://github.com/ElementAI/N-BEATS/blob/master/experiments/electricity/generic.gin}}), with a hidden size of 512. Note, N-BEATS is a univariate model and values presented here are multiplied by a factor of \(m\) to account for the multivariate data. Another dimension of comparison is the number of parameters used in the model. Demonstrated in \cref{tab:num-params}, fully connected models like N-BEATS, their number of parameters scales linearly with lookback window and forecast horizon length, while for Transformer-based and {\shortname}, the number of parameters remains constant.

\paragraph{N-HiTS} We use an N-HiTS model with hyperparameters as sugggested in their original paper (3 stacks, 1 block in each stack, 2 MLP layers, 512 hidden size). For the following hyperparameters which were not specified (subject to hyperparameter tuning), we set the pooling kernel size to \([2, 2, 2]\), and the number of stack coefficients to \([24, 12, 1]\). Similar to N-BEATS, N-HiTS is a univariate model, and values were multiplied by a factor of \(m\) to account for the multivariate data.

\begin{table}[H]
  \centering
  \caption{Number of parameters in each model across various lookback window and forecast horizon lengths. The models were instantiated for the ETTm2 multivariate dataset (this affects the embedding and projection layers in Autoformer). Values for N-HiTS in this table are \textbf{not} multiplied by \(m\) since it is a global model (i.e. a single univariate model is used for all dimensions of the time-series).}
    \begin{tabular}{c|r|rrr}
    \toprule
    \multicolumn{2}{c}{Methods} & \multicolumn{1}{c}{Autoformer} & \multicolumn{1}{c}{N-HiTS} & \multicolumn{1}{c}{{\shortname}} \\
    \midrule
    \multirow{8}[2]{*}{\begin{sideways}Lookback\end{sideways}} & 48    &   10,535,943  &           927,942  &          1,314,561  \\
          & 96    &   10,535,943  &        1,038,678  &          1,314,561  \\
          & 168   &   10,535,943  &        1,204,782  &          1,314,561  \\
          & 336   &   10,535,943  &        1,592,358  &          1,314,561  \\
          & 720   &   10,535,943  &        2,478,246  &          1,314,561  \\
          & 1440  &   10,535,943  &        4,139,286  &          1,314,561  \\
          & 2880  &   10,535,943  &        7,461,366  &          1,314,561  \\
          & 5760  &   10,535,943  &      14,105,526  &          1,314,561  \\
    \midrule
    \multirow{8}[2]{*}{\begin{sideways}Horizon\end{sideways}} & 48    &   10,535,943  &           927,942  &          1,314,561  \\
          & 96    &   10,535,943  &           955,644  &          1,314,561  \\
          & 168   &   10,535,943  &           997,197  &          1,314,561  \\
          & 336   &   10,535,943  &        1,094,154  &          1,314,561  \\
          & 720   &   10,535,943  &        1,315,770  &          1,314,561  \\
          & 1440  &   10,535,943  &        1,731,300  &          1,314,561  \\
          & 2880  &   10,535,943  &        2,562,360  &          1,314,561  \\
          & 5760  &   10,535,943  &        4,224,480  &          1,314,561  \\
    \bottomrule
    \end{tabular}%
  \label{tab:num-params}%
\end{table}%

%%%%%%%%%%%%%%%%%%%%%%%%%%%%%%%%%%%%%%%%%%%%%%%%%%%%%%%%%%%%%%%%%%%%%%%%%%%%%%%
%%%%%%%%%%%%%%%%%%%%%%%%%%%%%%%%%%%%%%%%%%%%%%%%%%%%%%%%%%%%%%%%%%%%%%%%%%%%%%%

\end{document}